\documentclass[a4paper,UKenglish, cleveref, autoref, thm-restate]{lipics-v2021}

\hideLIPIcs
\usepackage{amssymb}
\usepackage{amsmath,bm}
\usepackage{comment}
\usepackage{caption}
\usepackage{subcaption}
\usepackage{graphicx}
\usepackage{stackengine}
\usepackage{multirow}
\usepackage{colortbl}
\usepackage{cancel}
\usepackage{xspace}
\usepackage{booktabs}
\usepackage{stmaryrd}
\usepackage{paralist}
\usepackage{nicefrac}
\usepackage{pifont}
\usepackage{lipsum}

\usepackage{algorithm}
\usepackage[noend]{algpseudocode} 
\usepackage{tikz,ifthen,pgfplots}
\usepackage[colorinlistoftodos]{todonotes}

\usepackage{newfloat}
\usepackage{listings}
\usepackage{tikz,ifthen,pgfplots}
\usetikzlibrary{arrows,trees,backgrounds,automata,shapes,decorations,plotmarks,fit,calc,positioning,shadows,chains}
\tikzstyle{every pin edge}=[<-,shorten <=1pt]
\tikzstyle{neuron}=[circle,fill=black!25,minimum size=17pt,inner sep=0pt]
\tikzstyle{input neuron}=[neuron, fill=green!50]
\tikzstyle{output neuron}=[neuron, fill=red!50]
\tikzstyle{hidden neuron}=[neuron, fill=blue!50]
\tikzstyle{merged neuron}=[neuron, fill=orange!50]
\tikzstyle{annot} = [text width=6em, text centered]
\tikzstyle{nnedge} = [-{stealth},shorten >=0.1cm, shorten <=0.05cm,line width=0.8pt,black]
\tikzstyle{proofNode} = [rounded rectangle, fill=purple!30]
\tikzstyle{proofEdge} = [-{stealth},shorten >=0.1cm, shorten <=0.05cm,line width=0.8pt,black]
\usetikzlibrary{calc}

\newcommand{\relu}{\textit{ReLU}\xspace}
\newcommand{\originNetwork}{\ensuremath{N}\xspace}
\newcommand{\PreserveBackslash}[1]{\let\temp=\\#1\let\\=\temp}
\newcommand{\specTree}{\mathcal{T}}
\newcommand{\verifier}{\ensuremath{\mathtt{LpVerifier}}\xspace}

\newcommand{\atomicProp}{\ensuremath{\mathtt{AP}}\xspace}

\newcommand{\argmax}{\mathop{\rm arg~max}\limits}

\newcommand{\tool}{\ensuremath{\mathsf{Oliva}}\xspace}
\newcommand{\toolg}{\ensuremath{\tool^\mathsf{GR}}\xspace}
\newcommand{\toolb}{\ensuremath{\tool^\mathsf{SA}}\xspace}

\newcommand{\R}{\mathbb{R}}
\newcommand{\dnn}{\ensuremath{\originNetwork}\xspace}
\newcommand{\inputSpec}{\Phi}
\newcommand{\outputSpec}{\Psi}
\newcommand{\reluSpec}{\Gamma}
\newcommand{\reluAction}{r}
\newcommand{\reluActionPos}{r^+}
\newcommand{\reluActionNeg}{r^-}
\newcommand{\reluHeuristic}{\mathsf{H}}
\newcommand{\true}{\ensuremath{\mathsf{true}}\xspace}
\newcommand{\false}{\ensuremath{\mathsf{false}}\xspace}

\newcommand{\specDist}{\hat{p}}

\newcommand{\specDistMin}{\hat{p}_{\mathrm{min}}}

\newcommand{\ce}{\hat{\bm{x}}}

\newcommand{\cepo}{\textrm{CePO}\xspace}
\newcommand{\bab}{\ensuremath{\mathtt{BaB}}\xspace}
\newcommand{\reluInput}{\hat{x}}
\newcommand{\reluSet}{\mathsf{Set}}

\newcommand{\queue}{\ensuremath{Q}\xspace}
\newcolumntype{C}[1]{>{\PreserveBackslash\centering}p{#1}}
\newcommand{\partialOrder}{\sqsubset}

\newcommand{\reward}{\ensuremath{\mathsf{R}}\xspace}
\newcommand{\seman}[1]{\ensuremath{\mathsf{susp}(#1)}\xspace}

\newcommand{\abcrown}{{$\alpha\beta$-\textsf{Crown}}\xspace}
\newcommand{\marabou}{{\textsf{Marabou}}\xspace}
\newcommand{\neuralsat}{{\textsf{NeuralSAT}}\xspace}
\newcommand{\eran}{{\textsf{ERAN}}\xspace}

\newcommand{\fsb}{{\textsf{FSB}}\xspace}

\definecolor{mygreen}{RGB}{200, 255, 200}
\newcommand{\tbgreen}{\cellcolor{mygreen}}
\definecolor{myred}{RGB}{255, 200, 200}
\newcommand{\tbred}{\cellcolor{myred}}
\newcommand{\mnist}{\texttt{MNIST}\xspace}
\newcommand{\cifar}{\texttt{CIFAR-10}\xspace}

\newcommand{\OVAL}{\texttt{OVAL21}\xspace}
\newcommand{\base}{\texttt{BASE}}
\newcommand{\deep}{\texttt{DEEP}}
\newcommand{\wide}{\texttt{WIDE}}
\newcommand{\lfour}{\texttt{L4}}
\newcommand{\ltwo}{\texttt{L2}}

\newcommand{\descitem}[1][Item]{\item[#1~---]}%

\newcounter{mydefinition}

\newcounter{myexample}

\newcounter{mylemma}

\newcounter{myproposition}

\newcounter{myassumption}

\newcounter{myremark}

\title{Efficient Neural Network Verification via Order Leading Exploration of Branch-and-Bound Trees} %TODO Please add

\titlerunning{Efficient Neural Network Verification via Order Leading Exploration of \bab Trees} %TODO optional, please use if title is longer than one line

\author{Guanqin Zhang}{University of New South Wales, Sydney, Australia \and CSIRO's Data61, Sydney, Australia}{Guanqin.zhang@unsw.edu.au}{0000-0002-3844-8180}{CSIRO' Data61 PhD Scholarship}

\author{Kota Fukuda}{Kyushu University, Fukuoka, Japan}{fukuda.kota.527@s.kyushu-u.ac.jp}{0009-0000-0292-0896}{JST BOOST Grant No. JPMJBS2406}

\author{Zhenya Zhang}{Kyushu University, Fukuoka, Japan \and National Institute of Informatics, Tokyo, Japan}{zhang@ait.kyushu-u.ac.jp}{0000-0002-3854-9846}{JST BOOST Grant No. JPMJBY24D7 and JSPS Grant No. JP25K21179}

\author{H.M.N. Dilum Bandara}{CSIRO's Data61, Sydney, Australia \and University of New South Wales, Sydney, Australia}{dilum.bandara@data61.csiro.au}{0000-0002-2927-5628}{}

\author{Shiping Chen}{CSIRO's Data61, Sydney, Australia \and University of New South Wales, Sydney, Australia}{shiping.chen@data61.csiro.au}{0000-0002-4603-0024}{}

\author{Jianjun Zhao}{Kyushu University, Fukuoka, Japan}{zhao@ait.kyushu-u.ac.jp}{0000-0001-8083-4352}{JSPS KAKENHI Grant No. JP23H0337}

\author{Yulei Sui}{University of New South Wales, Sydney, Australia }{y.sui@unsw.edu.au}{0000-0002-9510-6574}{Australian Research Council Grants No. FT220100391 and No. DP250101396.}

\authorrunning{G. Zhang et al.} %TODO mandatory. First: Use abbreviated first/middle names. Second (only in severe cases): Use first author plus 'et al.'

\Copyright{Jane Open Access and Joan R. Public} %TODO mandatory, please use full first names. LIPIcs license is "CC-BY";  http://creativecommons.org/licenses/by/3.0/


 %TODO mandatory; please add comma-separated list of keywords

\category{} %optional, e.g. invited paper

\relatedversion{} %optional, e.g. full version hosted on arXiv, HAL, or other respository/website
\nolinenumbers %uncomment to disable line numbering

\EventEditors{Jonathan Aldrich and Alexandra Silva}
\EventNoEds{2}
\EventLongTitle{39th European Conference on Object-Oriented Programming (ECOOP 2025)}
\EventShortTitle{ECOOP 2025}
\EventAcronym{ECOOP}
\EventYear{2025}
\EventDate{June 30--July 2, 2025}
\EventLocation{Bergen, Norway}
\EventLogo{}
\SeriesVolume{333}
\ArticleNo{28}
\begin{document}

\maketitle

\begin{abstract}
The vulnerability of neural networks to adversarial perturbations has necessitated formal verification techniques that can rigorously certify the quality of neural networks. As the state-of-the-art, branch-and-bound (\bab) is a ``divide-and-conquer'' strategy that applies off-the-shelf verifiers to sub-problems for which they perform better. While \bab can identify the sub-problems that are necessary to be split, it explores the space of these sub-problems in a naive ``first-come-first-served'' manner, thereby suffering from an issue of inefficiency to reach a verification conclusion.
To bridge this gap, we introduce an order over different sub-problems produced by \bab, concerning with their different likelihoods of containing counterexamples. Based on this order, we propose a novel verification framework \tool that explores the sub-problem space by prioritizing those sub-problems that are more likely to find counterexamples, in order to efficiently reach the conclusion of the verification. Even if no counterexample can be found in any sub-problem, it only changes the order of visiting different sub-problems and so will not lead to a performance degradation.  Specifically, \tool has two variants, including \toolg, a greedy strategy that always prioritizes the sub-problems that are more likely to find counterexamples, and \toolb, a balanced strategy inspired by simulated annealing that gradually shifts from exploration to exploitation to locate the globally optimal sub-problems. 
We experimentally evaluate the performance of \tool on 690 verification problems spanning over 5 models with datasets \mnist and \cifar. 
Compared to the state-of-the-art approaches, we demonstrate the speedup of \tool for up to $25\times$ in \mnist, and up to $80\times$ in \cifar. 

\end{abstract}
\section{Introduction}
The rapid advancement of artificial intelligence (AI) has propelled the state-of-the-art across various fields, including computer vision, natural language processing, recommendation systems, etc. Recently, there is also a trend that adopts AI products in safety-critical systems, such as autonomous driving systems, in which neural networks are used in the perception module to perceive external environments. In this type of application, unexpected behaviors of neural networks can bring catastrophic consequences and intolerable social losses; given that neural networks are notoriously vulnerable to deliberate attacks or environmental perturbations~\cite{madry2017towards, goodfellow2015explaining}, effective quality assurance techniques are necessary in order to expose their defects timely before their deployment in the real world.

Formal verification is a rigorous approach that can ensure the quality of target systems by providing mathematical proofs on conformance of the systems with their desired properties. In the context of neural networks, formal verification aims to certify that a neural network, under specific input conditions, will never violate a pre-defined specification regarding its behavior, such as safety and robustness, thereby providing a rigorous guarantee that the neural network can function as expected in real-world applications. With the growing emphasis on AI safety, neural network verification has emerged as a prominent area of research, leading to the development of innovative methodologies and tools.

As a straightforward approach, exact encoding solves the neural network verification problem by encoding the inference process and specifications to be logical constraints and applying off-the-shelf or dedicated solvers, such as MILP solvers~\cite{tjeng2018evaluating} and SMT solvers~\cite{katz2017reluplex}, to solve the problem. However, this approach can be very slow due to the non-linearity of neural network inferences and so they are not scalable to neural networks of large sizes. In contrast, approximated methods~\cite{singh2019abstract,singh2018deepz,wang2018efficient} that over-approximate the reachable region of neural networks are more efficient: once the over-approximation satisfies the specification, the original output must also satisfy; however, if the over-approximation violates the specification, it does not indicate that the original output also violates, and in this case it may raise a \emph{false alarm}, i.e., a specification violation that actually does not exist.

To date, \emph{Branch-and-Bound} (\bab)~\cite{bunel2020branch} is the state-of-the-art verification approach that overarches multiple advanced verification tools, such as \abcrown~\cite{wang2021beta} and \marabou~\cite{wu2024marabou}. It is essentially a ``divide-and-conquer'' strategy, and is often jointly used with off-the-shelf approximated verifiers due to their great efficiency. Given a verification problem, it first applies an approximated verifier to the problem, and splits it to sub-problems if the verifier raises a false alarm, until all the sub-problems are verified or a real counterexample is detected, as a witness of specification violation. As application of approximated verifiers to sub-problems leads to less over-approximation, \bab can thus resolve the issue of false alarms of approximated verifiers and outperform their plain
application to the original problem. 

\paragraph*{Motivations}
As a ``divide-and-conquer'' strategy, \bab can produce a large space that consists of quantities of sub-problems, especially for verification tasks that are reasonably difficult. However, when exploring this space, the classic \bab adopts a naive ``first come, first served'' strategy, which ignores the importance of different sub-problems and thus is not efficient to reach a verification conclusion. 
Notably, different sub-problems produced by \bab are not equally important---with a part of the sub-problems, it is more likely to find a real counterexample that can show the violation of the specification, and thereby we can reach a conclusion for the verification efficiently without going through the remaining sub-problems.

\paragraph*{Contributions} 
In this paper, we propose a novel verification approach \tool, which is an \underline{\textbf{o}}rder \underline{\textbf{l}}eading \underline{\textbf{i}}ntelligent \underline{\textbf{v}}erification technique for \underline{\textbf{a}}rtificial neural networks. 

We first define an order called \emph{counterexample potentiality} over the different sub-problems produced by \bab, following our previous work~\cite{fukuda2025adaptive,olive}. Our order estimates \emph{how likely} a sub-problem is to contain a counterexample, based on the following two attributes: 
\begin{inparaenum}
    \item the level of problem splitting of the sub-problem, which implies how much approximation the approximated verifier needs to perform. The less approximation there is, the more likely the verifier can find a real counterexample;
    \item a quantity obtained by applying approximated verifiers to the sub-problem, which is an indicator of \emph{the degree} to which the neural network satisfies/violates the specification.
\end{inparaenum}
By these two attributes, we define the counterexample potentiality order over the sub-problems, as a proxy to suggest their likelihood of containing counterexamples and serve as guidance for our verification approach.

Then, we devise our approach \tool that exploits the order to explore the sub-problem space. In general, \tool favors the sub-problems that are more likely to contain counterexamples. Once it can find a real counterexample, it can immediately terminate the verification and return with a verdict that the specification can be violated; even if it cannot find such a counterexample, after visiting all the sub-problems, it can still manage to certify the neural network without a significant performance degradation. 

To exemplify the power of our order-guided exploration, we propose two variants of \tool: 
\begin{compactitem}
    \item \toolg employs a greedy exploitation strategy, which always prioritizes the sub-problems that have a higher likelihood of containing counterexamples. This approach focuses on the most suspicious regions of the sub-problem space, and is likely to quickly expose the counterexamples such that the verification can be concluded quickly; 
    \item To avoid overfitting of our proposed order, we also devise \toolb inspired by simulated annealing~\cite{kirkpatrick1983optimization}, the famous stochastic optimization technique. The approach works similarly to simulated annealing: it maintains a variable called \emph{temperature} that keeps decreasing throughout the process, which can control the trade-off between ``exploitation'' and ``exploration''. At the initial stage, the temperature is high, and so the algorithm allows more chances of exploring the sub-problems that are not promising; as the temperature goes down, it converges to the optimal sub-problem. As the exploration of the search space has been done at the early stage of the algorithm, it is likely to converge to the globally optimal sub-problems and thereby find counterexamples. In this way, we mitigate the potential issue of being too greedy, and aim to strike a balance between ``exploitation'' and ``exploration'' in the search for suspicious sub-problems. Compared to \emph{Monte Carlo tree search (MCTS)} adopted in~\cite{fukuda2025adaptive}, the stochasticity of simulated annealing offers the advantage that, while a fixed tree exploration policy in~\cite{fukuda2025adaptive} may fail to find counterexamples, repeating the experiments with different random seeds allows to explore the tree in a different way, thus increasing the chances of finding counterexamples.
\end{compactitem}

\paragraph*{Evaluation}
To evaluate the performance of \tool, we perform a large scale of experiments on 690 problem instances spanning over 5 neural network models associated with the commonly-used datasets \mnist and \cifar. 
By a comparison with the state-of-the-art verification approaches, we demonstrate the speedup of \tool, for up to $25\times$ in \mnist, and up to $80\times$ in \cifar. These experimental results involve the cases for which existing MCTS-based search~\cite{fukuda2025adaptive} does not work well, but our approach manages to find counterexamples thanks to repeated runs. Moreover, we also show the breakdown results of \tool for problem instances that are finally certified and the instances that are finally falsified. Experimental results show that \tool is particularly efficient for those instances that are falsified, which demonstrates the effectiveness of our approach.

For a verification problem whose result is unknown beforehand, it is always desired to reach the conclusion as quickly as possible. Given that verification of neural networks is typically time- and resource-consuming, our approach provides a meaningful way to accelerate the verification process.  While performing verification of neural networks with the aim of finding counterexamples sounds similar to approaches like testing or adversarial attacks, our approach differs fundamentally from those approaches, in the sense that, while those approaches deal with a single input each time and so they can never exhaust the search space, our approach deals with sub-problems that are finitely many, and so we can finally provide rigorous guarantees for specification satisfaction of neural networks.

\paragraph*{Paper organization} 
The rest of the paper is organized as follows: \S{}\ref{sec:overview} overviews our approach by using an illustrative example;  \S{}\ref{sec:preliminary} introduces the necessary technical background; \S{}\ref{sec:proposeidea} presents the proposed incremental verification approach; \S{}\ref{sec:evaluation} presents our experimental evaluation results; \S{}\ref{sec:literature} discusses related work. Conclusion and future work are presented in \S{}\ref{sec:conclusion}.

\section{Overview of The Proposed Approach}\label{sec:overview}
In this section, we use an example to illustrate how the proposed approach solves neural network verification problems.

\subsection{Verification Problem and \bab Approach}\label{sec:egclassicbab}

\begin{figure}[!tb]
    \centering
        \includegraphics[width=0.8\textwidth]{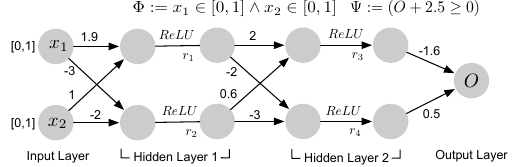}
    \caption{A neural network $\originNetwork$ and specification}
    \label{fig:network}
\end{figure}

Fig.~\ref{fig:network} depicts a (feed-forward) neural network $\originNetwork$. It has an input layer, an output layer, and two \emph{hidden layers} that are \emph{fully-connected}, namely, the output of each hidden layer is computed by taking the weighted sum of the output of the previous layer, and applying the ReLU activation function. The output $O$ of the neural network is computed by taking only the weighted sum of the output of the second hidden layer (without activation function). The weights of each layer are as labeled in Fig.~\ref{fig:network}. This neural network $\originNetwork$ is expected to satisfy such a specification: for any input $(x_1, x_2)\in [0,1]\times [0,1]$, it should hold that the output $O+2.5 \ge 0$. 
Verification aims to give a formal proof to certify that $\originNetwork$ indeed satisfies the specification, or give a counterexample instead if $\originNetwork$ does not satisfy the specification.

As the state-of-the-art neural network verification approach, Branch-and-Bound (\bab)~\cite{bunel2020branch} has overarched several famous verification tools, such as \abcrown~\cite{wang2021beta}. The application of \bab to neural network verification often relies on the combination with off-the-shelf verifiers, and a common choice involves the family of approximated verifiers. These verifiers can efficiently decide whether the neural network satisfies the specification, by constructing a convex over-approximation of neural network outputs: if the over-approximated output satisfies the specification, the original output must also satisfy; however, if not, it does not indicate that the original output also violates the specification, and so it may raise a \emph{false alarm} that reports a specification violation that is actually not existent. 

\begin{figure}[!tb]
\centering
    \begin{subfigure}[b]{0.7\textwidth}
    \centering
\includegraphics[width=\textwidth]{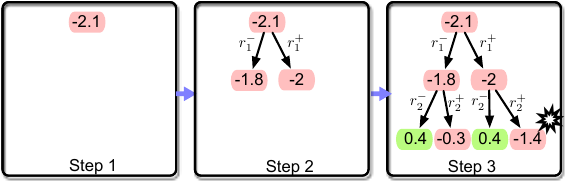}
    \caption{\bab's process for verification}
    \label{fig:babrun}
    \end{subfigure}
    
    \begin{subfigure}[b]{0.7\textwidth}
    \centering
        \includegraphics[width=\textwidth]{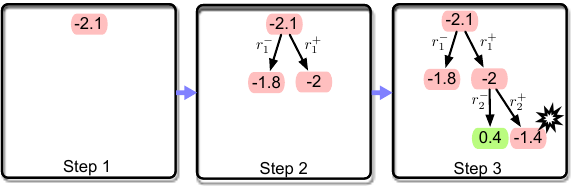}
    \caption{\tool's process for verification.}
    \label{fig:Oliverun}
    \end{subfigure}
    \caption{Neural network verification problem and solution via \bab (\bab \emph{vs.} \tool).}\label{fig:babvs}
\end{figure}

\bab is introduced to solve this problem. It is essentially a ``divide-and-conquer'' strategy that splits the problem when necessary and applies approximated verifiers to sub-problems to reduce the occurrences of false alarms. In this context, a problem can be split by predicating over the sign of the input of a ReLU function (i.e., by allowing the ReLU's input to be positive or negative only), such that it can be decomposed to two linear functions each of which is easier to handle.
\bab decides whether a (sub-)problem is needed to be split further, based on a value $\specDist$ returned from the approximated verifier. Intuitively, $\specDist$ indicates \emph{how far} the specification is from being violated by the over-approximation of a (sub-)problem. 
If $\specDist$ is positive, it implies that problem has been verified and so there is no need to split it; otherwise, it implies that the problem cannot be verified, and so \bab needs to check whether it is a false alarm, by checking whether the counterexample reported by the verifier is a spurious one. 
Fig.~\ref{fig:babrun} illustrates how \bab solves the verification problem in Fig.~\ref{fig:network}:

\begin{description}
    \descitem[Step 1] \bab first applies a verifier to the original verification problem, which returns a negative  $\specDist = -2.1$. By validating the counterexample reported by the verifier, \bab identifies that it is a false alarm and decides to split the problem;
    
    \descitem[Step 2]\label{step:secondStep} \bab splits the problem identified by the root node, and applies the verifier to the two sub-problems respectively. Again, it identifies that the verifier raises a false alarm for each of the sub-problems. Therefore, it needs to split both sub-problems;
    
    \descitem[Step 3] Similarly, \bab  applies the verifier in turn to the newly expanded sub-problems. It manages to verify two sub-problems (i.e., the two nodes both with $\specDist = 0.4$), and also identifies a false alarm (i.e., the node with $\specDist = -0.3$). In the node with $\specDist = -1.4$, \bab finds that the counterexample associated with this sub-problem is a real one; because having one such real counterexample suffices to show that the specification can be violated, \bab terminates the verification at this point with a verdict that the specification is violated.   
\end{description}    

\subsection{The Proposed Approach}
For a verification problem that is considerably difficult, \bab can often produce a space that contains a large number of sub-problems. However, as shown in Fig.~\ref{fig:babrun}, the classic \bab explores these sub-problems by a naive ``first come, first served'' order, which ignores the difference of the sub-problems in terms of their importance,  and thus can be very inefficient. 

To bridge this gap, we propose an approach that explores the space of sub-problems guided by an order of importance of different sub-problems. Specifically, we identify the importance of different sub-problems by their likelihood of containing counterexamples, in the sense that the more likely a sub-problem contains counterexamples, it should be more prioritized in the exploration of the sub-problem space---once we can find a counterexample there, we can immediately terminate the verification and draw a conclusion. Even if it does not manage to find a counterexample in any problem, it just visits the sub-problems in a different order from the original \bab, so it still would not be much slower.

Fig.~\ref{fig:Oliverun} demonstrates how our approach solves the same verification problem in Fig.~\ref{fig:network}. 

\begin{compactitem}
    \item In early stages, \tool works the same as \bab, i.e., it visits the (sub-)problems and obtains the same feedback $\specDist$ from the verifier, and based on that, it decides to split the (sub-)problems and apply verifiers to sub-problems;
    \item The difference comes from the rightmost plot of Fig.~\ref{fig:Oliverun}, in which \tool prioritizes the sub-problem  that has  $\specDist = -2$ rather than its sibling that has $\specDist = -1.8$ to expand. This is because $\specDist$ is a indicator that reflects how far the sub-problem is from being violated (based on the over-approximation performed by the verifier), and in this case $\specDist = -2$ signifies that this sub-problem has been violated \emph{more} than the other and so it is more likely to find a counterexample there. Indeed, \tool  manages to find a real counterexample in one of its sub-problems, thereby terminating the verification immediately after that. 
\end{compactitem}
By this strategy, we manage to save two visits to sub-problems compared to \bab, each of which consists of an expensive problem solving process, and therefore, it reaches the conclusion of the problem more efficiently than \bab.

\paragraph*{A variant of our approach inspired by simulated annealing} Essentially, our proposed approach changes the naive ``breadth-first'' exploration of the sub-problem space in the classic \bab to an intelligent way guided by the likelihood of finding counterexamples in different sub-problems. However, as this guidance consists of a total order over different branches, it may lose some chances of finding counterexamples in some branches that initially seem not promising. 

To mitigate this issue, we also propose a variant of our approach, inspired by simulated annealing~\cite{kirkpatrick1983optimization}. It works as follows: throughout the process, we maintain a temperature that keeps decreasing slowly, which is used to control the trade-off between ``exploitation'' and ``exploration''.
Initially, when the temperature is high, our algorithm tends to explore the sub-problem space by assigning a considerably high probability to accept a sub-problem that is not promising; as temperature goes down, it converges to the optimal sub-problem; due to the exploration of the search space at the initial stage, the algorithm is likely to converge to the globally optimal sub-problems. In this way, we can mitigate the side-effect of being too greedy, and strike a balance between ``exploitation'' and ``exploration''.  

\section{Preliminaries}\label{sec:preliminary}
In this section, we first introduce the neural network verification problem, and then introduce the state-of-the-art verification approach called \emph{branch-and-bound (\bab)}.

\subsection{Neural Network Verification Problem}
In this paper, we consider feed-forward neural networks, as depicted in Fig.~\ref{fig:network}.

\begin{definition}[Neural networks]
     A (feed-forward) neural network $\originNetwork: \mathbb{R}^n \to \mathbb{R}^m$ maps an $n$-dim input to an $m$-dim output (see Fig.~\ref{fig:network} for an example). 
It accomplishes the mapping by alternating between affine transformations (parametrized by weight matrix $W_i$ and bias vector $B_i$) and non-linear activation functions $\sigma$; namely, the output $\bm{x}_i$ of layer $i$ ($i\in\{1, \ldots, L\}$) is computed based on the output $\bm{x}_{i-1}$ of layer $i-1$, as follows:
\begin{math}
    \bm{x}_i = \sigma(W_i\bm{x}_{i-1} + B_i)
\end{math}. Here, the computation unit that computes each dimension of $\bm{x}_i$ is called a \emph{neuron}; in each layer $i$, the number of neurons is equivalent to the number of dimensions of  $\bm{x}_i$. Specially, $\bm{x}_0$ is the input and $\bm{x}_L$ is the output of the neural network. There could be different choices for the non-linear activation functions, such as ReLU, sigmoid and tanh;
following many existing works~\cite{liu2021algorithms}, we adopt Rectified Linear Unit (ReLU) (i.e., $\relu(x) = \max(0, x)$) as the activation function in our neural networks. 
\end{definition}

Specifications are logical expressions that formalize users' desired properties about neural networks. In this paper, we adopt the following notation in Def.~\ref{def:spec} as our specification formalism, and later we show that it can be used to formalize commonly-used properties, such as \emph{local robustness} of neural networks.
\begin{definition}[Specification]\label{def:spec}
    We denote by a pair $(\Phi, \Psi)$ a \emph{specification} for neural networks, where $\Phi$ is called an input specification that predicates over the input region, and $\Psi$ is called an output specification that predicates over the output of a neural network. Specifically, we denote the output specification $\Psi$ as follows: $f(\originNetwork(\bm{x}_0)) > 0$, where $f\colon \R^m \to \R$ is a function that maps an $m$-dim vector to a real number.
\end{definition}

\begin{definition}[Verification problem]
    A verification problem aims to answer the following question: given a neural network $\originNetwork$ and a specification $(\Phi, \Psi)$, whether $\Psi(\originNetwork(\bm{x}))$ holds, for any input $\bm{x}$ that holds $\Phi(\bm{x})$. A verifier is a tool used to answer a verification problem: it either returns $\true$, certifying $\originNetwork$'s satisfaction to the specification, or returns $\false$ with a \emph{counterexample} $\ce$, which is an input that holds $\Phi(\ce)$ but does not hold $\Psi(\ce)$.
\end{definition}

 We now explain how our notations can be used to formalize a neural network verification problem against local robustness, which is a property often considered in the domain of image classification. 
Local robustness requires that a neural network classifier should make consistent classification for two images $\bm{x}$ and $\bm{x}_0$, where $\bm{x}$ is produced by adding small perturbations to $\bm{x}_0$. Formally, 
given a reference input  $\bm{x}_0$, $\Phi(\bm{x})$ requires that the input $\bm{x}$ must stay in the region $\{\bm{x} \mid \|\bm{x} - \bm{x}_0\|_\infty \leq \epsilon \}$, where $\|\cdot\|_\infty$ denotes the $\ell^\infty$-norm distance metric and $\epsilon$ is a small real value, and $\Psi(\originNetwork(\bm{x}))$ requires that $\min_{1\le i \le m, i\neq i_l} (\originNetwork(\bm{x})_{i_l} - \originNetwork(\bm{x})_i) > 0$, where $\originNetwork(\bm{x})_i$ denotes the $i$-th component of the output vector $\originNetwork(\bm{x})$, and $i_l$ is the label of $\bm{x}_0$ inferred by the neural network, i.e., $i_l = \argmax_{1\le i\le m} \originNetwork(\bm{x}_0)_i$.

\subsection{Branch-and-Bound (\bab) -- State-of-the-Art Verification Approach}\label{sec:bab}
There have been various approaches proposed to solve the neural network verification problem. Exact encoding~\cite{katz2017reluplex,tjeng2018evaluating,cheng2017maximum} formalizes the inference process and the specification of a neural network to be an SMT or optimization problem and solves them by dedicated~\cite{katz2017reluplex} or off-the-shelf~\cite{tjeng2018evaluating, cheng2017maximum} solvers. These approaches are sound and complete, but due to the non-linearity of neural network inferences, they suffer from severe scalability issues and cannot handle neural networks of large sizes. To resolve this issue, approximated approaches~\cite{singh2018fast, singh2018deepz, singh2019abstract} over-approximate the output of neural networks by linear relaxation, and they can certify the satisfaction of specification if the over-approximated output satisfies the specification. These approaches are typically efficient and sound; however, they are not complete, i.e., they may raise false alarms with spurious counterexamples, if the over-approximated output does not satisfy the specification. As the state-of-the-art, Branch-and-Bound (\bab)~\cite{bunel2020branch} employs approximated verifiers and splits the problem when necessary (i.e., if a false alarm arises in a problem), because by solving sub-problems it can improve the precision of approximation and reduce the occurrence of false alarms. \bab keeps problem splitting until all of the sub-problems are certified (such that the original problem can be certified) or a real counterexample is detected with a sub-problem (such that the original problem is falsified). 

In the following, we explain the necessary ingredients and detailed workflow of \bab.

\paragraph*{Approximated verifiers}
    Approximated verifiers, denoted as \verifier, are a class of verifiers that solve a verification problem by computing an over-approximation of the output region of neural networks---if the over-approximation satisfies the specification, it implies that the original output region must also satisfy it. 
Typically, \verifier can return a tuple $\langle \specDist, \ce \rangle$, where $\specDist\in\R$ is called a \emph{verifier assessment}, which is a quantity that indicates \emph{the extent} to which the over-approximated output satisfies the specification. Formally, given an \verifier, it can construct a  region $\hat{\Omega}  \supseteq \{\originNetwork(\bm{x}) \mid \Phi(\bm{x}) = \true\}$, which over-approximates the output region of a neural network $\originNetwork$; then, \verifier computes  $\specDist$ as follows: $\specDist := \min_{\bm{y}\in \hat{\Omega}}f(\bm{y})$, where $f$ is as defined in Def.~\ref{def:spec}. If $\specDist$ is positive, it implies that the original output also satisfies the specification, and so it can certify the satisfaction of neural networks. Conversely, if $\specDist$ is negative, \verifier deems that the specification is violated and provides a counterexample input $\ce$ as an evidence. However, due to the over-approximation, this reported violation of \verifier could be a false alarm and the counterexample $\ce$ may be a spurious one, i.e., the corresponding output $\originNetwork(\ce)$ of $\ce$ actually satisfies the specification. In this case, just based on the result of \verifier, we cannot infer whether $\originNetwork$ holds or violates the specification. This situation is referred to as the \emph{completeness issue} of approximated verifiers.

There have been various approximation strategies that can be used to implement approximated verifiers, e.g.,  \emph{DeepPoly}~\cite{singh2019abstract}  and \emph{ReluVal}~\cite{wang2018formal}, and mostly, the over-approximation is accomplished by using linear constraints to bound the output range of the non-linear ReLU functions. In particular, we assume that our adopted approximated verifiers hold the \emph{monotonicity} property, namely, for two output regions of neural networks that hold $\Omega_1\subset\Omega_2$, if we over-approximate them by using the same approximated verifier \verifier, the obtained over-approximation $\hat{\Omega}_1$ and  $\hat{\Omega}_2$ hold that $\hat{\Omega}_1 \subseteq \hat{\Omega}_2$. While there can be different options of over-approximation strategies, our proposed approach is orthogonal to them, so we can adopt any approximated verifiers that hold our assumption about monotonicity.

\paragraph*{Branch-and-Bound (\bab)}
\bab is the state-of-the-art neural network verification approach, and has been adopted in several advanced verification tools, such as \abcrown~\cite{wang2021beta} and \marabou~\cite{wu2024marabou}. It is essentially a ``divide-and-conquer'' strategy, that divides a verification problem adaptively and applies off-the-shelf verifiers (such as  approximated verifiers) to solve the sub-problems. Because approximated verifiers often achieve better precision on sub-problems, \bab can thus overcome the weaknesses of the plain application of approximated verifiers to the original problem and resolve the issues of false alarms.

Before looking into the details of \bab, we first introduce ReLU specification, which is an important notion in the algorithm of \bab. 

\begin{definition}[ReLU specification] Let \dnn be a neural network consisting of $K$ neurons, and  $\reluInput_i\in\R$ ($i\in\{1, \ldots, K\}$) be the input for the ReLU function in the $i$-th neuron. An \emph{atomic proposition} $\atomicProp$ w.r.t.  the $i$-th neuron  is defined as either $\reluInput_i \ge 0$ (written as $\reluActionPos_i$) or $\reluInput_i < 0$ (written as $\reluActionNeg_i$). Then, a ReLU specification $\reluSpec$ is defined as the conjunction of a sequence $\reluSet(\reluSpec) := \{\atomicProp_1, \ldots, \atomicProp_{|\reluSpec|}\}$ of atomic propositions, where each $\atomicProp\in\reluSet(\reluSpec)$ is defined w.r.t. a distinct neuron. 
$|\reluSpec|$ is the number of atomic propositions in $\reluSet(\reluSpec)$; specially if $|\reluSpec| = 0$, $\reluSpec$ is denoted as $\top$. Moreover,
we  define a \emph{refinement relation} $\prec$ over ReLU specifications, namely, given two ReLU specifications $\reluSpec_i$ and $\reluSpec_j$, we say $\reluSpec_j$ refines $\reluSpec_i$  (denoted as $\reluSpec_i \prec \reluSpec_j$) if and only if 
$\reluSet(\reluSpec_i) \subset \reluSet(\reluSpec_j)$.
\end{definition}

By selecting a neuron $i$, we can split the ReLU function into two linear functions, each identified by a predicate $\reluActionPos_i$ or $\reluActionNeg_i$ over the input of ReLU. Thereby, a neural network verification problem boils down to two sub-problems, for each of which \verifier does not need to perform over-approximation for the ReLU function in the $i$-th neuron.

\begin{algorithm}[!tb]
\caption{Branch-and-Bound (\bab)~\cite{bunel2020branch}}
\label{alg:bab}
\begin{algorithmic}[1]
\Require A neural network $\dnn$, an input specification $\inputSpec$, an output specification $\outputSpec$, an approximated verifier $\verifier(\cdot)$, and a ReLU selection heuristic $\reluHeuristic(\cdot)$. 
\Ensure A $\mathit{verdict}\in\{\true, \false\}$.
\State $\queue \gets \{\top\}$ \label{line:initQueue} 
\State $\specTree\gets \emptyset$ \label{line:initialSpecTree} 

\Function{BaB}{$\dnn, \inputSpec, \outputSpec, \queue$}
\label{line:functionBaB}
\If{$\Call{Empty}{\queue}$}
\State \Return \true \label{line:Qempty} 
\EndIf
\State $\reluSpec \gets \Call{Pop}{\queue}$ \label{line:popQ} 
\State $\langle\specDist, \ce\rangle\gets \verifier(\dnn, \inputSpec, \outputSpec, \reluSpec)$ \label{line:callAppVerifier} 
\State $\specTree\gets\specTree\cup \{\langle\reluSpec, \specDist\rangle\}$ \label{line:treeRecord} 
\If{$\specDist < 0$} 
\If{$\Call{Valid}{\ce, \dnn, \outputSpec}$} \label{line:validCE} 
\State \Return $\false$ \label{line:babFalse} 
\Else \label{line:invalidCE}
\State $\reluAction_k \gets \reluHeuristic(\reluSpec)$ \label{line:selectNextRelu}
\For{$a \in \{\reluActionPos_k, \reluActionNeg_k\}$}
\State $\Call{Push}{Q, \reluSpec\land a}$ \label{line:push}
\EndFor
\EndIf
\EndIf
\State  \Return $\Call{BaB}{\dnn, \inputSpec, \outputSpec, \queue}$ \label{line:recursiveCallBaB} 
\EndFunction

\end{algorithmic}
\end{algorithm}

The workflow of \bab is presented in 
Alg.~\ref{alg:bab}. In Alg.~\ref{alg:bab}, we allow \verifier to take an additional argument, namely, a ReLU specification $\reluSpec$, which identifies a sub-problem by adding the constraints in $\reluSet(\reluSpec)$ to constrain the inputs of a number of selected ReLU functions. It uses a FIFO queue \queue to maintain the problem to be solved, which is initialized to be a set that consists of $\top$ only, identifying the original verification problem. 
\begin{compactenum}[i)]
    \item First, it applies \verifier to the original problem (Line~\ref{line:callAppVerifier}): if \verifier returns a positive $\specDist$, or a negative $\specDist$ with a valid counterexample $\ce$ (Line~\ref{line:validCE}), then verification can be terminated with a verdict returned accordingly; 
    \item \label{step:case} In the case it returns a negative $\specDist$ with a spurious counterexample (Line~\ref{line:invalidCE}), it divides the verification problem into two sub-problems. This is achieved by first selecting a neuron (i.e., a ReLU) in the network according to a pre-defined ReLU selection heuristics $\reluHeuristic$ (Line~\ref{line:selectNextRelu}), and then identifying two sub-problems each identified by an additional constraint on the input of the selected ReLU function (Line~\ref{line:push}). 
    \item It applies \verifier to the new sub-problems respectively, and decides whether it needs to further split the sub-problem, following the same rule in Step~\ref{step:case}. In Alg.~\ref{alg:bab}, this is implemented by recursive call of the \textsc{BaB} function in Line~\ref{line:recursiveCallBaB}.
\end{compactenum}

In \bab, the ReLU selection heuristics $\reluHeuristic$ involves an order over different neurons such that it can select the next ReLU based on an existing ReLU specification. There has been a rich body of literature that proposes different ReLU selection strategies, such as DeepSplit~\cite{henriksen2021deepsplit} and FSB~\cite{de2021improved}. However, our proposed approach in~\S{}\ref{sec:proposeidea} is orthogonal to these strategies and so it can work with existing strategies.  In this work, we follow an existing neural network verification approach~\cite{ugare2023incremental} and we adopt the state-of-the-art ReLU selection strategy in~\cite{henriksen2021deepsplit}.

As a result of Alg.~\ref{alg:bab}, it forms a binary tree $\specTree$ (see Line~\ref{line:initialSpecTree} and Line~\ref{line:treeRecord}) that records the history of problem splitting during the \bab process. In this tree, each node $\langle\reluSpec, \specDist\rangle$ denotes a sub-problem, identified by a ReLU specification $\reluSpec$ and the verifier assessment $\specDist$ returned by \verifier that signifies the satisfaction of the sub-problem.

\begin{lemma}[Soundness and completeness]\label{lem:bab}
    The \bab algorithm is sound and complete.
\end{lemma}

\begin{proof}
    Soundness requires that if \bab returns \true, the neural network must satisfy the specification. 
    The soundness of \bab relies on that: 1) \verifier is sound; 2) the ReLU specifications identified by the leaf nodes of the \bab tree cover all the cases about the input conditions of the split ReLUs in the neural network.

    Completeness requires that if \bab returns \false, the neural network must violate the specification. The completeness of \bab relies on the fact that the counterexamples $\ce$ reported by \bab are all validated and so they must be real. Moreover, in the worst case if all neurons are split, the sub-problem will be linear and so a $\ce$ must be a real one if it exists.

    Finally, \bab is guaranteed to return either \true or \false within a finite time budget, because the number of neurons in a neural network is finite. 
\end{proof}

\section{\tool: The Proposed Verification Approach}\label{sec:proposeidea}

As a ``divide-and-conquer'' strategy, \bab can produce a huge space that consists of quantities of sub-problems; however, as shown in Alg.~\ref{alg:bab}, the existing \bab approach explores this space in a naive ``first come, first served'' manner (implemented by the \emph{first-in-first-out} queue in Alg.~\ref{alg:bab} that stores the (sub)-problems to be solved), which can be very inefficient to exhaust the sub-problem space. 

Our proposed approach involves exploring the \bab tree in an intelligent fashion, guided by the severity of different tree nodes. 
In~\S{}\ref{sec:cepo}, we showcase such a severity order, defined by the probability of finding counterexamples with a sub-problem. The intuition behind this order is that, by prioritizing sub-problems that are more likely to find counterexamples, we may quickly find a counterexample and thereby immediately terminate the verification. If we cannot find it  after visiting all sub-problems, we achieve certification of the problem.

\subsection{Counterexample Potentiality Order}\label{sec:cepo}
We introduce an order, previously proposed in~\cite{fukuda2025adaptive,olive}, over different sub-problems based on their probability of containing counterexamples. Given a node in \bab tree that identifies a sub-problem, we can infer this probability using the following two attributes:
\begin{description}
\item[\emph{Node depth.}]
In a \bab tree, the depth of a node signifies the levels of problem splitting, and for more finely-split sub-problems,  \verifier introduces less over-approximation. Because of this, if a node $\reluSpec$ with a greater depth is still deemed by \verifier as violating the specification (i.e., the verifier assessment $\specDist$ is negative), it is more likely that $\reluSpec$ indeed contains real counterexamples. 
%Given that a node in \bab tree is split only if its $\specDist$ by \verifier is negative (i.e., it is not verified yet), the deeper such a node is, the more precise $\specDist$ is, and so the more likely that such a node contains a real counterexample.

\item[\emph{Verifier assessment.}] Given a node in \bab tree that identifies a sub-problem, its $\specDist$ returned by \verifier can be considered as a quantitative indicator of \emph{how far} the sub-problem is from being violated. Due to our assumption about the monotonicity of \verifier in performing over-approximation, given a fixed \verifier,  $\specDist$ is correlated to the original output region of the neural network. Therefore, in the case $\specDist$ is negative, the greater $|\specDist|$ is, there is a higher possibility that the original output region is closer to violation of the specification, and so it is more likely that  the sub-problem contains a real counterexample.
%Under a fixed approximation strategy \verifier,  $\specDist$ is computed with a correlation to the original output region of the network. Therefore, in the case $\specDist$ is negative, the greater $|\specDist|$ is, the more likely that  the sub-problem contains a real counterexample.
\end{description}

Based on the above two node attributes, we  define the suspiciousness of a \bab tree node, and then define a Counterexample POtentiality (\cepo) order over different nodes.

\begin{definition}[Suspiciousness of sub-problems]\label{def:susp}
Let $\reluSpec$ be a node of the \bab tree that has a verifier assessment $\specDist$ (with a counterexample $\ce$ if $\specDist < 0$). The suspiciousness $\seman{\reluSpec}\in [0,1]\cup \{+\infty, -\infty\}$ of the node $\reluSpec$ 
maps $\reluSpec$ to a real number as follows:
\begin{align*}
    \seman{\reluSpec} := \begin{cases}
        -\infty  & \text{if } \specDist > 0 \\
        +\infty  & \text{if } \specDist<0 \text{ and } \mathtt{valid}(\ce) \\
        \lambda\frac{\mathtt{depth}(\reluSpec)}{K} + (1-\lambda)\frac{\specDist}{\specDistMin} & \text{otherwise}
    \end{cases}
\end{align*}
where $\lambda\in [0,1]$ is a parameter that controls the weights of the two attributes, and $K$ is the total number of neurons (i.e., ReLUs) in the network.
\end{definition}

Intuitively, the suspiciousness of a node encompasses a heuristic that estimates the probability of the relevant sub-problem violating the specification. It is particularly meaningful in the case when $\specDist<0$ and the counterexample is spurious, and set to be $+\infty$ if the sub-problem is provably violated, and $-\infty$ if the sub-problem is certified. By their suspiciousness, we define the \cepo order  over different nodes as follows:

\begin{definition}
[Counterexample potentiality (\cepo) order] \label{def:cePotentialOrder}
    Let $\reluSpec_1$ and $\reluSpec_2$ be two nodes in a \bab tree, and $\specDist_1$ and $\specDist_2$ be verifier assessments for $\reluSpec_1$ and $\reluSpec_2$, respectively. We define a \cepo order $\partialOrder$ between the two nodes as follows: 
    \begin{align*}
        \reluSpec_1 \partialOrder \reluSpec_2 \quad\textit{iff}\quad \seman{\reluSpec_1} < \seman{\reluSpec_2}
    \end{align*}
\end{definition}
The \cepo order allows us to sort the nodes in \bab tree,\footnote{In the case if two nodes have the same suspiciousness, we simply impose a random order; in the case both are $-\infty$ or $+\infty$, the comparison is meaningless and so we do not need to sort them.} and so it can serve as a guidance to our verification approach in~\S{}\ref{sec:greedy} and~\S{}\ref{sec:balanced}.

\subsection{\toolg: Greedy Exploration of \bab Tree}\label{sec:greedy}
\begin{algorithm}[!tb]
\caption{\toolg: The proposed greedy algorithm}
\label{alg:dfs_nn_verify}
\begin{algorithmic}[1] %[1] enables line numbers
\Require A neural network $\originNetwork$, input and output specification $\Phi$ and $\Psi$, an approximated verifier $\verifier(\cdot)$, a ReLU selection heuristic $\reluHeuristic(\cdot)$, and a hyperparameter $\lambda$.
\Ensure A $\mathit{verdict} \in \{\true, \false, \mathtt{timeout}\}$ 

\Statex
\State $\specTree\gets \{\varepsilon\}$
\State $\langle\specDist, \ce\rangle\gets \verifier(\originNetwork, \Phi, \Psi, \varepsilon)$ \label{line:appToOrigin}
\State $\reward(\varepsilon)\gets \seman{\varepsilon}$ \Comment{a metric for child selection}

\If{$\specDist<0$ and $\mathtt{not\;valid}(\ce)$}
\While{not reach termination condition}\label{line:whileLoopGreedy}
\State \Call{GreedyBaB}{$\varepsilon, \originNetwork, \Phi, \Psi$}\label{line:callGreedyBaBRoot}
\EndWhile
\State\label{line:returnGreedy}\Return
$
\begin{cases}
    \true & \text{if } \reward(\varepsilon) = -\infty \\
    \false & \text{if } \reward(\varepsilon) = +\infty \\
    \mathtt{timeout} & \text{otherwise}
\end{cases}
$
\Else
\State \label{line:originReturnGreedy}\Return
$
\begin{cases}
    \true & \text{if }\specDist > 0 \\
    \false & \text{if } \specDist < 0 \text{ and } \mathtt{valid}(\ce)
\end{cases}
$
\EndIf
\Statex
\Function{GreedyBaB}{$\reluSpec, \originNetwork, \Phi, \Psi$}
\State $\reluAction_k\gets \reluHeuristic(\reluSpec)$ \label{line:reluSelectGreedy}
\If{$\reluSpec \cdot \reluActionPos_k \in \specTree$} \Comment{check existence of children}\label{line:checkExistsGreedy}
\State $\reluSpec^*\gets \reluSpec\cdot a^*$\quad s.t. $a^* \gets \argmax_{a\in\{\reluActionPos_k, \reluActionNeg_k\}}{\reward(\reluSpec\cdot a)}$
\State \Call{GreedyBaB}{$\reluSpec^*, \originNetwork, \Phi, \Psi$}\label{line:recursiveCallGreedy}
\Else
\For{$a\in\{\reluActionPos_k, \reluActionNeg_k\}$}\label{line:expandChildrenGreedy}
    \Comment{$\specTree$ expansion via BaB.}
    \State $\langle\specDist, \ce\rangle \gets \verifier(\originNetwork, \Phi, \Psi, \reluSpec \cdot a)$ \label{line:verifierSubTreeGreedy}
    \Comment{apply \verifier with $\reluSpec\land a$}

    \State $\reward(\reluSpec\cdot a) \gets    \seman{\reluSpec\cdot a}$ \label{line:rewardSubtreeGreedy}
    \Comment{compute $\reward$}
 
    \State $\specTree \gets \specTree \cup \{\reluSpec\cdot a\}$ \label{line:updateTreeGreedy}
    \Comment{add to the tree of $\reluSpec\land a$}
    
\EndFor
\EndIf
\State $\reward(\reluSpec)\gets \argmax_{a\in\{\reluActionPos_k, \reluActionNeg_k \}}{\reward({\reluSpec\cdot a})}$\label{line:backpropagateGreedy}

\EndFunction
\end{algorithmic}
\end{algorithm}

The \toolg algorithm, presented in Alg.~\ref{alg:dfs_nn_verify}, implements a greedy strategy that explores the space of \bab tree nodes guided by the \cepo order. Compared to the classic \bab, it always selects the most suspicious node to expand, such that it can maximize the probability of finding counterexamples and thereby conclude the verification problem efficiently.

%The algorithm maintains a specification tree $\specTree$ initialized with an empty specification $\varepsilon$ (Line 1), representing the root of the verification problem with no additional constraints.
The algorithm begins with applying the approximated verifier \verifier to the original verification problem, identified by the ReLU specification $\varepsilon$ (Line~\ref{line:appToOrigin}), which returns a tuple $\langle\specDist, \ce\rangle$ containing a verifier assessment $\specDist$, possibly followed by a counterexample $\ce$. The suspiciousness $\seman{\varepsilon}$ of the root node is then computed based on this assessment. At this point, if $\specDist>0$ or if $\specDist<0$ with a valid counterexample $\ce$, the algorithm can be immediately terminated with a conclusive verdict (Line~\ref{line:originReturnGreedy}). Otherwise, the algorithm enters its main loop where it iteratively calls \textsc{GreedyBaB} to split and verify the sub-problems until a termination condition is reached (Line~\ref{line:whileLoopGreedy}-\ref{line:returnGreedy}). We elaborate on the termination condition later.

The \textsc{GreedyBaB} function implements the main process of tree exploration and problem splitting. By  each of its execution, it selects the maximal sub-problem in terms of the \cepo order in the tree and applies \verifier to its subsequent sub-problems with the aim of finding counterexamples. The function begins with selecting a ReLU $\reluAction_k$ as a successor of the current node (Line~\ref{line:reluSelectGreedy}) using the pre-defined ReLU selection heuristics $\reluHeuristic$ (see~\S{}\ref{sec:bab}), and then it recursively calls \textsc{GreedyBaB} until it reaches the greatest node $\reluSpec$, in terms of the \cepo order, whose children are not expanded yet (Line~\ref{line:checkExistsGreedy}-\ref{line:recursiveCallGreedy}). After reaching such a node $\reluSpec$, it expands the children of $\reluSpec$, by applying \verifier respectively to the two sub-problems identified by the children of $\reluSpec$ (Line~\ref{line:expandChildrenGreedy}-\ref{line:verifierSubTreeGreedy}). It also computes and records the suspiciousness of the newly expanded children (Line~\ref{line:rewardSubtreeGreedy}), and updates $\specTree$ that keeps track of the tree (Line~\ref{line:updateTreeGreedy}). Lastly, the greater suspiciousness over the children are propagated backwards to ancestor nodes until the root (Line~\ref{line:backpropagateGreedy}), serving as a reference for future node selection.

\paragraph*{Termination condition} 
The main loop in Line~\ref{line:whileLoopGreedy} of Alg.~\ref{alg:dfs_nn_verify} can be terminated on the satisfaction of any of the following three conditions:
\begin{compactitem}
    \item $\reward(\varepsilon) = -\infty$: This implies that all leaf nodes have been verified successfully (i.e., $\specDist > 0$ for all leaves, otherwise by Line~\ref{line:backpropagateGreedy} $\reward(\varepsilon)$ cannot be $-\infty$), allowing the algorithm to return a $\true$ verdict that certifies the specification satisfaction;
    \item $\reward(\varepsilon) = +\infty$: This implies that a valid counterexample has been found in some leaf node, and so by Line~\ref{line:backpropagateGreedy} its suspiciousness of $+\infty$ can be back-propagated to the root node; in this case, the algorithm can be terminated with a $\false$ verdict;
    \item $\mathtt{timeout}$: This occurs when the algorithm exceeds its allocated time budget without reaching either of the above conclusive verdicts.
\end{compactitem}
The three termination conditions correspond to the three cases of return in Line~\ref{line:returnGreedy}.
%These termination conditions are explicitly handled in Line~\ref{line:return} of the algorithm.

\subsection{\toolb: Simulated-Annealing-Style Exploration of \bab Tree}\label{sec:balanced}

\begin{algorithm}[!tb]
\caption{\toolb: The proposed simulated annealing algorithm}
\label{alg:sa_nn_verify}
\begin{algorithmic}[1] %[1] enables line numbers
\Require A neural network $\originNetwork$, input and output specification $\Phi$ and $\Psi$, an approximated verifier $\verifier(\cdot)$, a ReLU selection heuristic $\reluHeuristic(\cdot)$, and a hyperparameter $\lambda$.
\Ensure A $\mathit{verdict} \in \{\true, \false, \mathtt{timeout}\}$ 

\Statex
\State $\specTree\gets \{\varepsilon\}$
\State $\langle\specDist, \ce\rangle\gets \verifier(\originNetwork, \Phi, \Psi, \varepsilon)$ \label{line:appToOriginSA}
\State $\reward(\varepsilon)\gets \seman{\varepsilon}$ \Comment{a metric for child selection}
\State $T\gets  T_{max}$  \Comment{temperature}\label{line:initTempSA}

\If{$\specDist<0$ and $\mathtt{not\;valid}(\ce)$} \label{line:enterMainLoopSA}
\While{not reach termination condition}\label{line:whileLoopSA}
\State $T \gets \alpha \cdot T$ \label{line:tempChangeSA} \Comment{$T$ is decreased by $(1-\alpha)T$ in each iteration}
\State \Call{SABaB}{$\varepsilon, \originNetwork, \Phi, \Psi$}\label{line:callSABaBRoot}
\EndWhile
\State\label{line:return}\Return
$
\begin{cases}
    \true & \text{if } \reward(\varepsilon) = -\infty \\
    \false & \text{if } \reward(\varepsilon) = +\infty \\
    \mathtt{timeout} & \text{otherwise}
\end{cases}
$
\Else
\State \label{line:originReturnSA}\Return
$
\begin{cases}
    \true & \text{if }\specDist > 0 \\
    \false & \text{if } \specDist < 0 \text{ and } \mathtt{valid}(\ce)
\end{cases}
$
\EndIf
\Statex
\Function{SABaB}{$\reluSpec, \originNetwork, \Phi, \Psi$}
\State $\reluAction_k\gets \reluHeuristic(\reluSpec)$ \label{line:reluSelectSA}
\If{$\reluSpec \cdot r_k^+ \in \mathcal{T}$} \label{line:childrenExpandedSA}
        \State $\Delta p \gets \exp\left(\frac{\min \reward(\reluSpec \cdot a) - \max \reward(\reluSpec \cdot a)}{T}\right)$ \ s.t. $a\in \{r_k^+, r_k^-\}$\label{line:acceptSA}
        \State $\reluSpec^* \gets \reluSpec \cdot a^* $ \ s.t. $ a^*\gets
        \begin{cases}
            \text{randomly choose } r_k^+ \text{ or } r_k^- & \text{if } \mathtt{rand}(0,1) < \Delta p \\
            \argmax_{a \in \{r_k^+, r_k^-\}} \reward(\reluSpec \cdot a) & \text{otherwise}
        \end{cases}$ \label{line:deltaP}
\State \Call{SABaB}{$\reluSpec^*, \originNetwork, \Phi, \Psi$}\label{line:recursiveCallSA}
\Else
\For{$a\in\{\reluActionPos_k, \reluActionNeg_k\}$}\label{line:expandChildrenSA}
    \Comment{$\specTree$ expansion via BaB.}
    \State $\langle\specDist, \ce\rangle \gets \verifier(\originNetwork, \Phi, \Psi, \reluSpec \cdot a)$ \label{line:verifierSubTreeSA}
    \Comment{apply \verifier with $\reluSpec\land a$}
    \State $\reward(\reluSpec\cdot a) \gets    \seman{\reluSpec\cdot a}$ \label{line:rewardSubtreeSA}
    \Comment{compute $\reward$}
    \State $\specTree \gets \specTree \cup \{\reluSpec\cdot a\}$ \label{line:updateTreeSA}
    \Comment{add to the tree of $\reluSpec\land a$}
    
\EndFor
\EndIf
\State $\reward(\reluSpec)\gets \argmax_{a\in\{\reluActionPos_k, \reluActionNeg_k \}}{\reward({\reluSpec\cdot a})}$\label{line:backpropagateSA}
\EndFunction
\end{algorithmic}
\end{algorithm}

The greedy strategy in~\S{}\ref{sec:greedy} exploits the \cepo order in the exploration of \bab tree, so it can efficiently move towards the sub-problems that are more likely to find counterexamples. However, as it is a greedy strategy, in the case if the \cepo order is not sufficiently precise, the verification process can be easily trapped into a local optimum, and miss the chances of finding counterexamples in other branches than the one suggested by the \cepo order.

To bridge this gap, we propose an approach that adapts the classic framework of \emph{simulated annealing}, which not only follows the \cepo order, but also takes other branches into account. As a consequence, during the verification, it strikes a balance between ``exploitation'' of the suspicious branches and ``exploration'' of the less suspicious branches. 

\paragraph*{``Hill climbing'' vs. ``Simulated annealing''} Hill climbing is a stochastic optimization technique that aims to find the optimum of a black-box function. Due to the black-box nature of the objective function, it relies on sampling in the search space and selects only the samples that optimize the objective function as the direction to move. Therefore, this is a greedy strategy, similarly to our proposed approach \toolg in~\S{}\ref{sec:greedy}. While \toolg does not need to sample the search space, it can select the sub-problem that is more promising, and thereby move towards the direction that is more likely to achieve the objective. As a consequence, they both suffer from the issue of ``local optima''. 

In the field of stochastic optimization, \emph{simulated annealing}~\cite{kirkpatrick1983optimization} is an effective approach to mitigate the issue of ``local optima''  in hill climbing. It is inspired by the process of \emph{annealing} in metallurgy, in which the temperature is initially high but slowly decreases such that the physical properties of metals can be stabilized. In simulated annealing, there is also a temperature that slowly decreases throughout the process: when the temperature is initially high, it tends to explore the search space and assigns a considerably high probability to accept a sample that is not the optimal; as temperature goes down, it converges to the optimum that is likely to be the global one, thanks to the exploration of the search space at the initial stage.

 Our proposed approach \toolb adapts simulated annealing to our problem setting, and the core idea involves that, at the initial stage, we allow more chances of exploring the branches that are less promising. Then, after we have comprehensively explored the search space and obtained the information about the suspiciousness of different branches, we tend to exploit the branches that are more suspicious, namely, that are more likely to contain counterexamples. 

 \paragraph*{Algorithm details} The \toolb algorithm is presented in Alg.~\ref{alg:sa_nn_verify}. It also starts with checking the original verification problem (Line~\ref{line:appToOriginSA}), and enters the loop of tree exploration if the original problem cannot be solved by \verifier (Line~\ref{line:enterMainLoopSA}). However, compared to \toolb, it has a notable difference about the adoption of a temperature $T$ (Line~\ref{line:initTempSA}) which is a global variable that keeps decreasing in each loop of the algorithm (Line~\ref{line:tempChangeSA}).

In the function \textsc{SABaB}, the temperature $T$ is used when selecting the nodes to proceed. 
In the case if the children of the current node $\reluSpec$ have been expanded (Line~\ref{line:childrenExpandedSA}), unlike \toolg that always prefers the most suspicious child, \toolb selects a child according to the policy adapted from the original simulated annealing:
\begin{compactitem}
    \item It first computes an acceptance probability $\Delta p$ (Line~\ref{line:acceptSA}), by which it determines whether the selection of a child that is less promising is acceptable. This probability $\Delta p$ is decided by both the difference of ``energy'' (defined by the suspiciousness difference between two children) and the temperature $T$;
    \item In the original simulated annealing, a sample that is more promising can be accepted in any case, and a sample that is less promising can be accepted with the probability $\Delta p$. In our context, we adapt this policy as follows (Line~\ref{line:deltaP}):
    \begin{compactitem}
        \item If a random value in $[0,1]$ is less than $\Delta p$, we randomly select a child from the two children of $\reluSpec$ with the same probability, despite the \cepo order over the two children;
        \item Otherwise, we select the more suspicious child following the \cepo order. 
    \end{compactitem}
\end{compactitem}
The selected child will be used as the argument for the recursive call of \textsc{SABaB}, in order to proceed towards a node whose children are not expanded. After achieving such a node $\reluSpec$, it expands the children of $\reluSpec$ (Lines~\ref{line:expandChildrenSA}-\ref{line:updateTreeSA}) by applying \verifier to each child of $\reluSpec$, recording their suspiciousness (Line~\ref{line:rewardSubtreeSA}) and updating the \bab tree (Line~\ref{line:updateTreeSA}). Lastly, it back-propagates the greater suspiciousness over the children until the root (Line~\ref{line:backpropagateSA}), similarly to \toolb.

This gradual transition from exploration to exploitation helps \toolb to avoid premature convergence while ensuring an eventual focus on the promising branches. The termination conditions remain the same as in \toolg, checking for full verification ($\reward(\varepsilon) = -\infty$), counterexample discovery ($\reward(\varepsilon) = +\infty$), or timeout. However, by pursuing a higher coverage of the sub-problem space before being greedy, \toolb increases the possibility of discovering counterexamples that might be missed by the purely greedy strategy of \toolg.

\begin{theorem}
Both \toolg and \toolb are sound and complete.
\end{theorem}

\begin{proof}
    The proofs of soundness and completeness of our approaches are similar to that of \bab in Lemma~\ref{lem:bab}, so we skip the details. Intuitively, our approaches only introduce an order of visiting different nodes in \bab tree, but hold all the conditions that are necessary for the soundness and completeness of \bab. 
\end{proof}

\section{Experimental Evaluation} \label{sec:evaluation}

\subsection{Experiment Settings}
\paragraph*{Baselines and Metrics} 
We compare the two versions of \tool, namely, \toolg and \toolb, with three state-of-the-art neural network verification tools \abcrown~\cite{zhang2018efficient,wang2021beta}, \neuralsat~\cite{duong2023dpll}, and \bab-baseline, as presented in~\S{}\ref{sec:bab}. The settings of baseline approaches are detailed below:
\begin{itemize}
    \item \bab-baseline is implemented based on the \eran framework \cite{singh2018deepz, singh2019abstract}, which employs LP-based triangle relaxation for bounding the output of ReLU functions, and utilizes DeepSplit~\cite{henriksen2021deepsplit} as the ReLU selection heuristics for selecting ReLU function to split, i.e., $\reluHeuristic$ in Alg.~\ref{alg:bab}, Alg.~\ref{alg:dfs_nn_verify}, and Alg.~\ref{alg:sa_nn_verify}. 
    \item \abcrown~\cite{zhang2018efficient,wang2021beta} is applied with its default branch-and-bound settings. Moreover, a balanced strategy kFSB from \fsb \cite{de2021improved} is selected as the ReLU selection heuristics for selecting the ReLU function to split. 
    
    \item \neuralsat~\cite{duong2023dpll} is applied with its default DPLL(T) framework and we select the stabilized optimization as the neuron stability heuristics for ReLU activation pattern search. Additionally, it is registered with a random attack~\cite{das2021fast, yu2016derivative} to reject the easily detected violation instances at the early stage of verification processes.
\end{itemize} 
For our approaches, the hyperparameters are set as follows: $\lambda = 0.5$ (see Def.~\ref{def:susp}), $T_{max} = 1$ (see Line~\ref{line:initTempSA} of Alg.~\ref{alg:sa_nn_verify}) and $\alpha = 0.99$ (see Line~\ref{line:tempChangeSA} of Alg.~\ref{alg:sa_nn_verify}); in RQ3, we study the impact of these hyperparameters on the performance of our approaches.

We apply each of the five approaches (including two proposed approaches and three baseline approaches) to each of the verification problems, and we adopt 1000 secs as our time budget, following the VNN-COMP competition~\cite{muller2022third}. If an approach manages to solve the problem (either verifies the problem or reports a real counterexample), we deem this run as a \emph{solved} verification process, and record the time cost for reaching the verification conclusion. In particular, our \toolb is a stochastic approach, namely, its performance is subject to randomness. To compare its performance with other approaches, in Table~\ref{tab:comparisonWithBaB}, Table~\ref{tab:pairwisecase} and Fig.~\ref{fig:speedup}, we only show the performance of one random run; in Fig.~\ref{fig:figure_dis_at}, we particularly study the influence of randomness to \toolb, by repeating each verification process for 5 times.

Our evaluation metrics include the number of instances solved by an approach and the time costs for each successful verification process. For comparison of two different approaches, we also compute the speedup rate, which is the ratio of the time costs of the two approaches.

Moreover, in order to understand whether our counterexample potentiality order works, we compare the performances of our tools with the baseline approach, respectively for the problems that are finally certified and the problems that are finally falsified. We further study the influence of the hyperparameters (including $\lambda$ in Def.~\ref{def:susp} and $\alpha$ that decides the change rate of temperature in Line~\ref{line:tempChangeSA} of Alg.~\ref{alg:sa_nn_verify}) to the performances of our approaches. 
In the implementation of \tool, we adopt the same approximated verifiers and ReLU selection heuristics as the \bab-baseline. The code and experimental data of \tool are available online\footnote{\url{https://github.com/DeepLearningVerification/Oliva}}.

\begin{table}[!tb]
    \centering
    \caption{Benchmark details for the evaluation of verification.}\label{tab:benchmark}
\small
    \begin{tabular}{cccccc}
    \toprule
        Model & Architecture & Dataset & \#Activations &\# Instances & \#Images\\\midrule
        $\mnist_{{{\ltwo}}}$ &  2 $\times$ 256 linear  & \mnist & 512 & 100 &70\\ 
        $\mnist_{{{\lfour}}}$ &  4 $\times$ 256 linear  & \mnist & 1024 &78 &52\\ \hline
        % $\mnist_{{{\lsix}}}$ &  6 $\times$ 256 linear & \mnist & 1536 &99\\\midrule 
        $\OVAL_{{\base}}$  & 2 Conv, 2 linear  & \cifar & 3172 & 173 & 53\\  
        $\OVAL_{{\wide}}$  & 2 Conv, 2 linear  & \cifar & 6244 & 196 &53\\ 
        $\OVAL_{{\deep}}$  & 4 Conv, 2 linear & \cifar & 6756 & 143 &40\\ \bottomrule
    \end{tabular}
    \label{tab:models}
\end{table}

\paragraph*{Datasets and Neural Networks}
Our experimental evaluation uses two well-known datasets: \mnist, featuring images of handwritten digits for classification, and \cifar, featuring images of various real-world objects like airplanes, cars, and animals, with networks tasked to identify each class. These datasets are standard benchmarks that have been widely used in the neural network verification community and adopted in the VNN-COMP~\cite{muller2022third}, an annual competition in the community for comparing the performances of different verification tools. We evaluate two networks trained on \mnist that have fully-connected layers only, and three neural networks trained on \cifar that have both convolutional layer and fully-connected layers, with different network architectures, following common evaluation utilized in \eran~\cite{muller2022prima,wang2021beta} and \OVAL~\cite{bunel2018unified,bunel2020branch} benchmarks. More details of these benchmarks are presented in Table~\ref{tab:benchmark}.

\begin{figure}[!tb]
    \centering
    \includegraphics[width=0.6\linewidth]{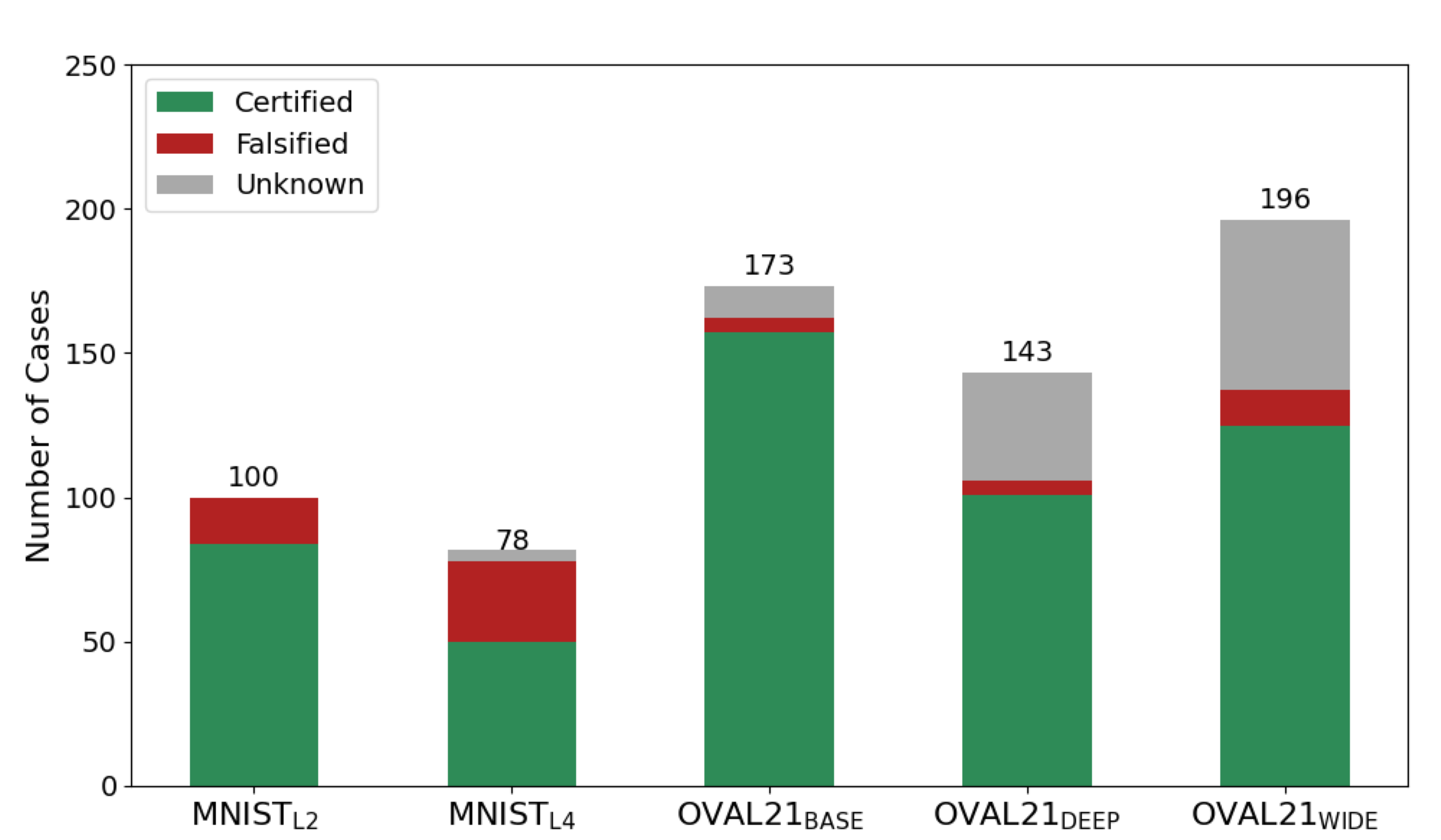}
    \caption{The distribution of Certified/Falsified/Unknown cases generated by \bab-baseline.}
    \label{fig:dis_boxplot_cond}
\end{figure}

\paragraph*{Specifications} In our experiments, all our neural network models are used for image classification tasks, and all the specifications adopted are about local robustness of the neural network models. Each input specification defines an  $\epsilon$ as the threshold for perturbation; and each output specification requires the expected label to be matched by the original input of the model. In total, we collected 690 problem instances. Table~\ref{tab:models} under the column ``\# Instances'' displays the number of instances in each model, and ``\# Images'' shows the number of different images covered by the different instances. Fig.~\ref{fig:dis_boxplot_cond} demonstrates the distribution of verification verdicts across the five models in our benchmark set, according to our experimental results. The green portion denotes tasks eventually verified (``Certified''), the red portion denotes tasks that were shown to violate the property (``Falsified''), and the gray portion corresponds to tasks that remained inconclusive (``Unknown'').

In order to present a meaningful comparison, we need to avoid the verification problems that are too simple, for which the problem can be solved within a small number of problem splitting and so there is no need to expand the \bab tree too much. To that end, we perform a selection of parameters (i.e., $\epsilon$ in Def.~\ref{def:spec}) of input specifications, from a range of $0/255$ to $16/255$ under the $L_{\infty}$ norm.  
Our approach involves a binary search-like algorithm to determine the proper perturbation values for each image, as follows:
\begin{compactenum}
    \item\label{item:1} Initially, we set $0/255$ as the lower bound $l$ and $16/255$ as the upper bound $u$, and calculates the midpoint by taking $m = \frac{l+u}{2}$; 
    \item\label{item:2} We apply \bab-baseline to the verification problem with $m$ as the $\epsilon$, and check the number of nodes in  the \bab tree. If the tree size is greater than 1, we accept it as a candidate parameter and move to the next image; otherwise, based on the results of verification, we proceed to update $m$ as like Step~\ref{item:1} and repeat Step~\ref{item:2} as follows: 
\begin{compactitem}
          \item  If the specification is violated, it implies that $\epsilon$ is too large and so we set $u$ to be $m$, and update $m$ accordingly;
          \item If the specification is satisfied, it implies that $\epsilon$ is too small and so we set $l$ to be $m$, and update $m$ accordingly.
    \end{compactitem}
\end{compactenum}
This process continues until a pre-defined budget is exhausted.

\begin{figure}[!tb]
    \centering
    \includegraphics[width=0.5\linewidth]{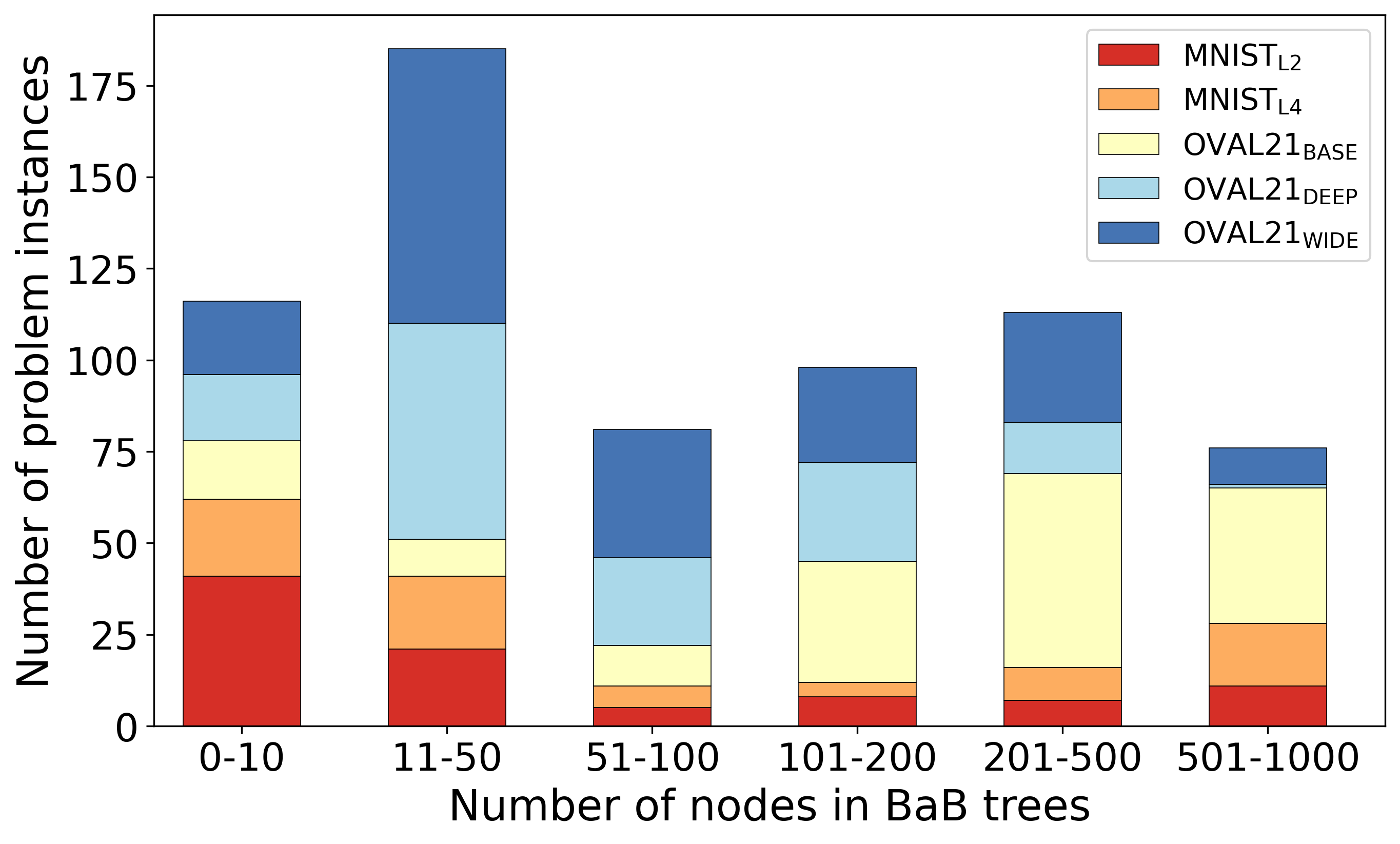}
    \caption{The distribution of the tree sizes generated by \bab-baseline.}
    \label{fig:dis_boxplot}
\end{figure}
Fig.~\ref{fig:dis_boxplot} shows the distribution of node counts across all models for the verification instances. The distribution confirms that all instances involve multiple tree nodes, ensuring meaningful sub-problem selection. Notably, more than half of the tree sizes fall within the range of 50-1000 nodes, highlighting the importance of effective sub-problem selection. The complexity of most instances necessitates careful choice in verification instances.

\paragraph*{Software and Hardware Setup}
All experiments were conducted on an AWS EC2 instance running a Linux system with 16GB of memory and an 8-core Xeon E5 2.90 GHz CPU. All tools are developed in Python 3.9, and we used the GUROBI solver 9.1.2~\cite{gurobi} for the LP-based optimization. For each verification instance, we set a time budget as 1000 seconds which is consistent with VNN-COMP~\cite{muller2022third}. 

\subsection{Evaluation}

\paragraph*{RQ1: Is \tool more efficient than existing approaches?}

In Table~\ref{tab:comparisonWithBaB}, we compare the performance of our approach with the state-of-the-art baseline approaches, \bab-baseline, \abcrown and \neuralsat, across all our verification problems. 

We observe that, \tool shows evident improvement for  benchmarks including $\OVAL_{{\base}}$, $\OVAL_{{\deep}}$ and $\OVAL_{{\wide}}$. On the $\OVAL_{{\base}}$, \toolb solves 159 instances in an average of 155.29 seconds, far surpassing \bab-baseline (42 instances in 770.7 seconds) and \abcrown (58 instances in 641.7 seconds). For $\OVAL_{\deep}$, \toolg solves 92 instances, compared to 33 by \bab-baseline and 55 by \abcrown, while using less time (250.72 seconds compared to 552.24 seconds). The difference is even more pronounced for $\OVAL_{{\wide}}$, where \toolg solves 131 instances in an average of 240.16 seconds, greatly exceeding both \bab-baseline (40 instances in 733.73 seconds) and \abcrown (63 instances in 557.51 seconds).

For $\mnist_{{\ltwo}}$ benchmarks that have simpler architectures compared to \OVAL, \bab-baseline performs well with 96 instances solved, while \toolg achieves 99 solved instances with a faster average time (57.76 seconds). As the size of the network increases to $\mnist_{{\lfour}}$, \abcrown declines to 44 instances, while \toolg achieves 67 solved instances with an average time of 146.31 seconds.
These results highlight the effectiveness and efficiency of \tool, compared to the baseline approaches.

\begin{table}[!tb]
    \centering
    
        \caption{RQ1 -- Overall comparison of different verification approaches, in terms of the number of solved problem instances and the time costs (in secs).}
    \label{tab:comparisonWithBaB}
\scalebox{0.85}{
    \begin{tabular}{l|rc|rc|rc|rc|rc }
    \toprule
        Model  &  \multicolumn{2}{c}{\bab-baseline}  & \multicolumn{2}{c}{\abcrown} & \multicolumn{2}{c}{Neuralsat} & 
        \multicolumn{2}{c}{\toolg} & \multicolumn{2}{c}{\toolb} \\ 
        
               & Solved  &Time  & Solved & Time  & Solved &Time  & Solved &Time & Solved &Time\\\hline

$\mnist_{{{\ltwo}}}$  &96&126.41 &87&\tbgreen51.32 &99&32.37 &95&96.79 &99&57.76  \\
$\mnist_{{{\lfour}}}$  &65&194.74 &44&428.03 &54&392.04 &\tbgreen67&146.31 &53&\tbgreen142.72  \\
$\OVAL_{{\base}}$ &42&770.7 &58&641.7 &70&621.21 &154&184.96 &\tbgreen159&\tbgreen155.29  \\
$\OVAL_{{\deep}}$ &33&694.74 &55&552.24 &55&539.59 &\tbgreen92&\tbgreen250.72 &87&261.68  \\
$\OVAL_{{\wide}}$ &40&733.73 &63&557.51 &65&533.01 &\tbgreen131&\tbgreen240.16 &112&288.0  \\

\hline
        \bottomrule
    \end{tabular}}

\end{table}

\begin{table}[!tb]
    \centering
    \caption{RQ1 -- Pairwise comparison on the number of additional solved problem instances from all verification tasks. The number in each cell implies the number of problem instance solved by the approach of the row, but not solved by the approach of the column.}
    \begin{tabular}{|c |c| c| c |c |c| }
    \toprule
     & \bab-Baseline & \abcrown & \neuralsat & \toolg & \toolb \\ \hline

\bab-Baseline  & 0 & 80 &59 &\tbred8& \tbred13\\ \hline
\abcrown & 111 & 0 & 11 &\tbred23& \tbred39\\ \hline
\neuralsat & 126 &47 & 0  &\tbred30& \tbred46\\\hline
\toolg & \tbgreen271 &\tbgreen255 &\tbgreen226& 0 &  40\\\hline
\toolb & \tbgreen247 &\tbgreen242 &\tbgreen213& 11 & 0  \\\hline
        
        \bottomrule
        
    \end{tabular}
    
    \label{tab:pairwisecase}
\end{table}

Table~\ref{tab:pairwisecase} presents a pairwise comparison, in which each cell indicates the number of problem instances solved by the approach listed in the row, but not solved by the approach listed in the column. Namely, we can perform pairwise comparison between two approaches by comparing the values in a pair of cells symmetric to the diagonal of Table~\ref{tab:pairwisecase}.
Compared to baseline approaches, our approaches  demonstrate superior performances, as indicated by the green area that shows the number of additional problems solved by our approaches by not by baseline approaches; in comparison, the numbers in the red area, that includes the number of problems solved by baseline approaches, but not solved by our approaches, are much less. 
By comparing the performances between \toolg and \toolb, we find that \toolg slightly outperforms \toolb, in that it solves 40 additional problems than \toolb, although there are 11 problems \toolb solves but \toolg does not do. That said, it also demonstrates the complementary strengths between different approaches, namely, there is no one global optimal approach that can solve all the problems, and so it is  worthwhile to try different approaches when one approach does not work.

In Fig.~\ref{fig:speedup}, we take a further comparison between \tool and \bab-baseline, and draw a scatter plot to show the individual instances for which \tool outperform \bab-baseline. The $x$-axis represents the time taken by the \bab-baseline method in seconds, while the y-axis shows the speedup ratio of \tool over \bab-baseline. This ratio is calculated as $\nicefrac{\tau_{\bab}(i)}{\tau_{\tool}(i)}$, where $\tau_{\bab}(i)$ and $\tau_{\tool}(i)$ denote the time taken by \bab-baseline and \tool on the instance $i$, respectively.  The blue dots and orange crosses represent all of the individual problem instances. The red line is the threshold at 1$\times$ speedup, above which the ratio distinguishes that \tool outperforms than \bab-baseline.

\begin{figure}[!tb]
    \centering
    \begin{subfigure}[b]{0.48\linewidth}
    \includegraphics[width=\linewidth]{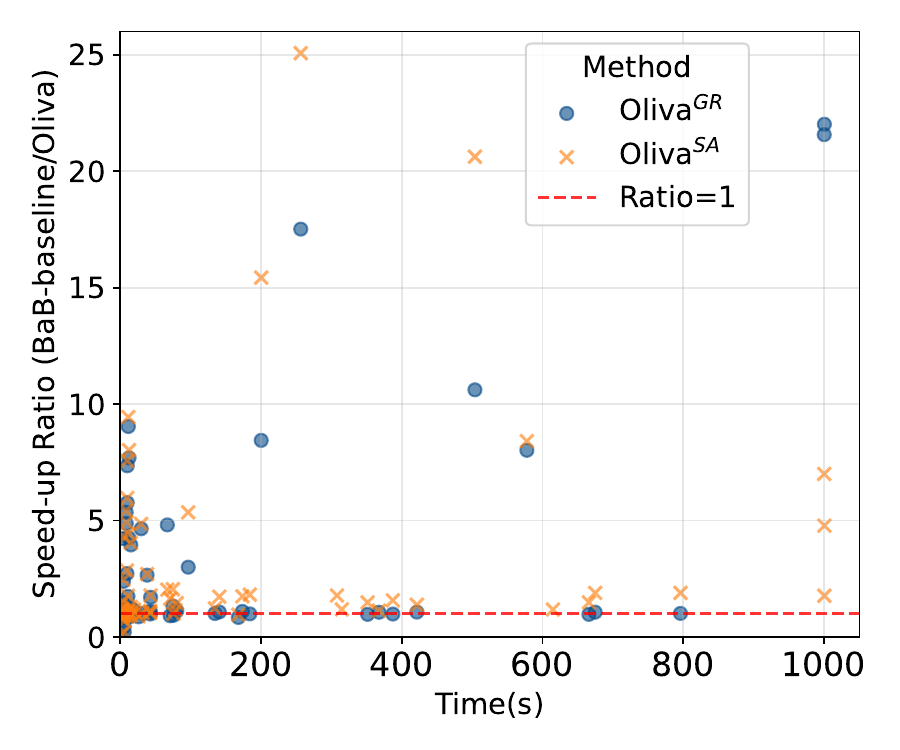}
        \caption{$\mnist_{\ltwo}$}
    \end{subfigure}
    \begin{subfigure}[b]{0.48\linewidth}
    \includegraphics[width=\linewidth]{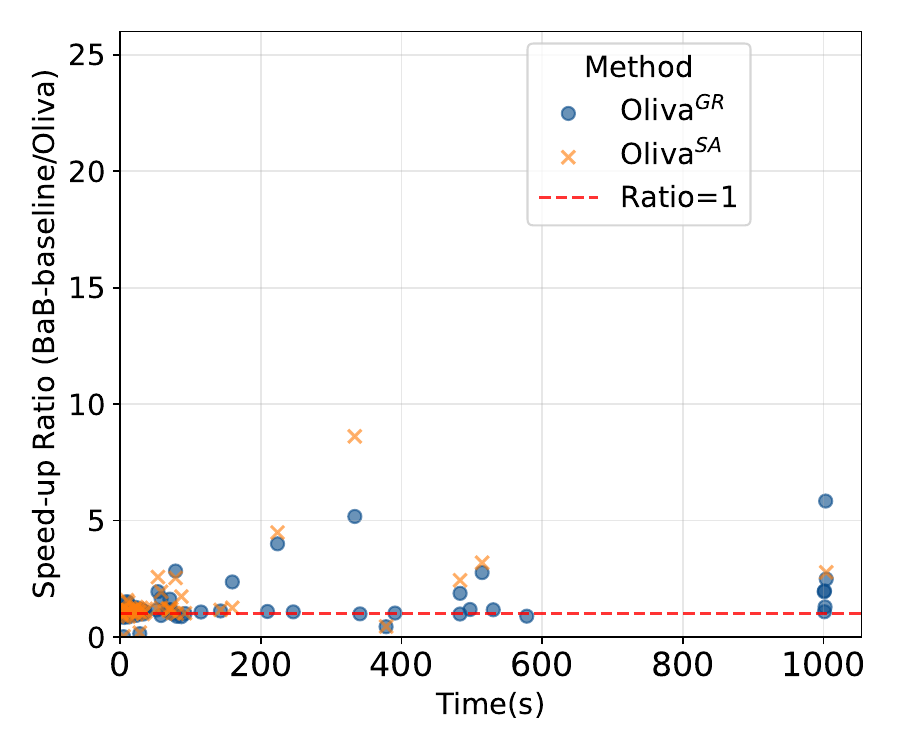}
        \caption{$\mnist_{\lfour}$}
    \end{subfigure}
    
    \begin{subfigure}[b]{0.48\linewidth}
    \centering\includegraphics[width=\linewidth]{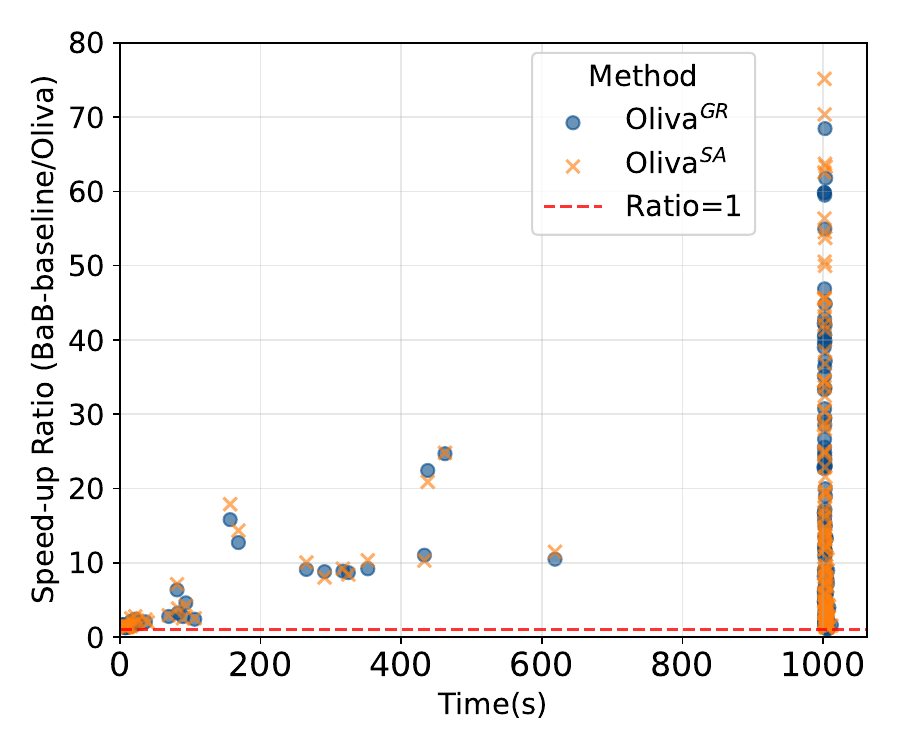}
        \caption{$\OVAL_{{\base}}$}
    \end{subfigure}
    \begin{subfigure}[b]{0.48\linewidth}
    \includegraphics[width=\linewidth]{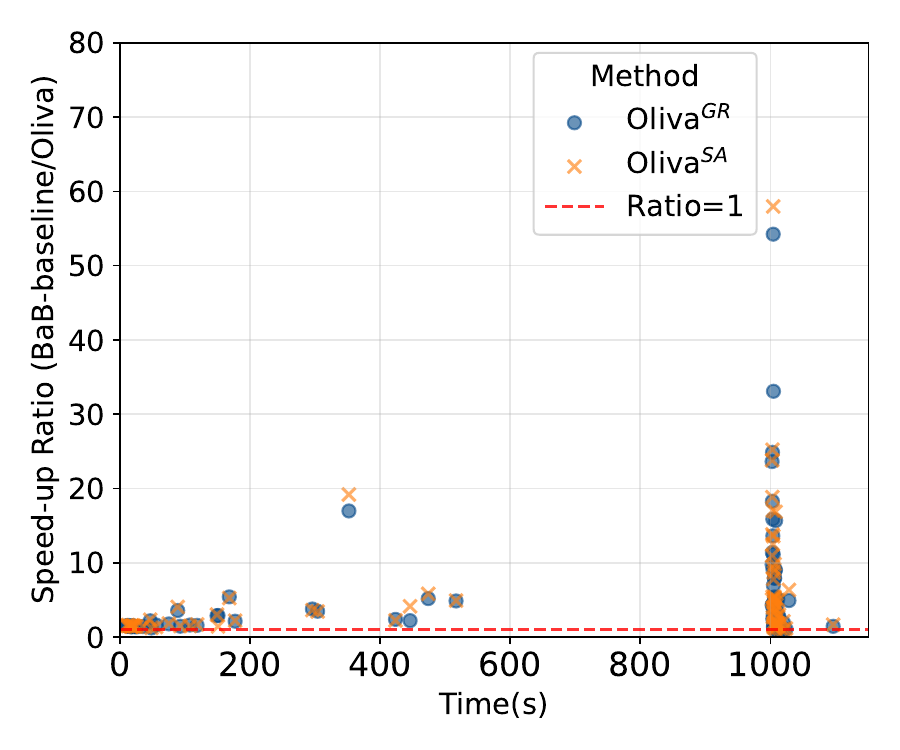}
        \caption{$\OVAL_{{\deep}}$}
    \end{subfigure}
    
    \begin{subfigure}[b]{0.48\linewidth}
        \includegraphics[width=\linewidth]{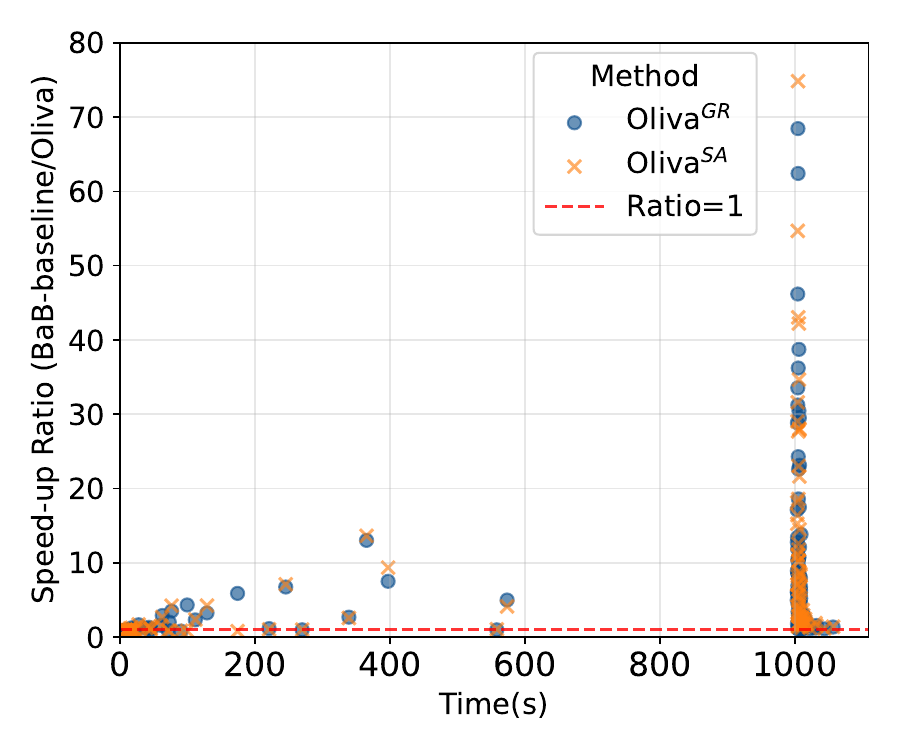}
        \caption{$\OVAL_{{\wide}}$}
    \end{subfigure}
    \caption{RQ1 -- Comparison of time cost and speedup ratio of the two variants \toolg and \toolb of our proposed approach \tool over \bab-baseline.}
    \label{fig:speedup}
\end{figure}

Across all the models, it can be seen that a significant number of blue dots (\toolg) and orange crosses (\toolb) appear above the red line, indicating that \tool outperforms the \bab-baseline. Notably, around the 1000-second mark on the $x$-axis, there are multiple instances where the speedup ratio is up to $80\times$, meaning that \tool can solve these problems more than 80 times faster than the \bab-baseline. This trend is particularly evident in \OVAL models ($\OVAL_{{\base}}$, $\OVAL_{{\deep}}$, and $\OVAL_{{\wide}}$) that have relatively complex architectures.

\begin{table}[!tb]
\caption{RQ1 -- Statistics of speedup ratios of two variants of \tool over \bab-baseline.}
    \label{tab:sspd}
    \centering
    \begin{tabular}{cccccc}
    \toprule
       Model  & Tool name & Min &Max & Median & Mean \\\midrule
        \multirow{2}{*}{Overall} & \toolg & 0.02&80.97 & 2.21& 7.27\\ 

        &  \toolb & 0.03&75.13&2.18&7.57 \\ \hline \midrule

        \multirow{2}{*}{$\mnist_\ltwo$}	& \toolg &	0.11	&22.02	&1.04&	2.49 \\

        & \toolb & 0.11	& 25.07 &	1.12&	2.50 \\\midrule

        \multirow{2}{*}{$\mnist_\lfour$} & \toolg & 0.02	& 5.84	& 1.10	& 1.38 \\ 

        & \toolb & 0.03	 &8.62	&1.17	&1.47 \\\midrule

        \multirow{2}{*}{$\OVAL_{{\base}}$} &\toolg & 1.07	&80.97	&6.93 &	13.73 \\

        & \toolb & 1.28	&75.13	&7.02	&13.96 \\\midrule

        \multirow{2}{*}{$\OVAL_{{\deep}}$} & \toolg & 1.01	& 54.23	& 2.80	&5.37 \\

        & \toolb & 1.03	& 57.96	& 2.34	&5.31\\\midrule

        \multirow{2}{*}{$\OVAL_{{\wide}}$} & \toolg & 0.81	&68.44	&2.95	&7.48  \\

        & \toolb & 0.78	&74.83&	2.55	&7.60\\\bottomrule
        
    \end{tabular}
    
\end{table}
In Table~\ref{tab:sspd}, we summarize the statistical information of speedup ratios of our approaches over \bab-baseline. Overall, the two variants of \tool achieve consistent performance improvements, with median speedups exceeding 2 times and mean gain values around 7 times.

\begin{figure}[!tb]
\centering
    \begin{subfigure}[b]{0.32\linewidth}
        \includegraphics[width=\linewidth]{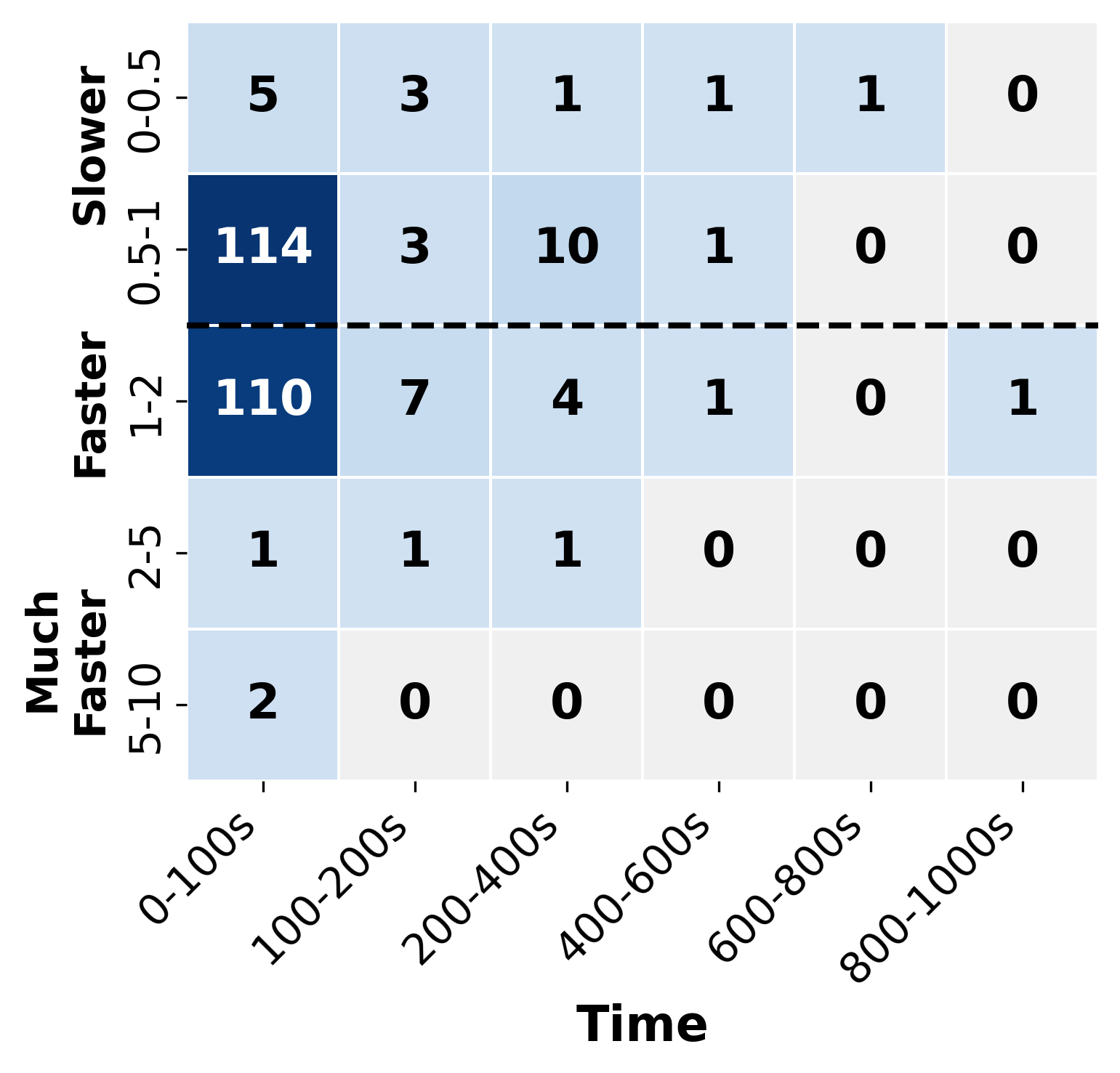}
        \caption{SA attempt 1}
        \label{fig:figure_at1}
    \end{subfigure}
    \begin{subfigure}[b]{0.32\linewidth}
        \includegraphics[width=\linewidth]{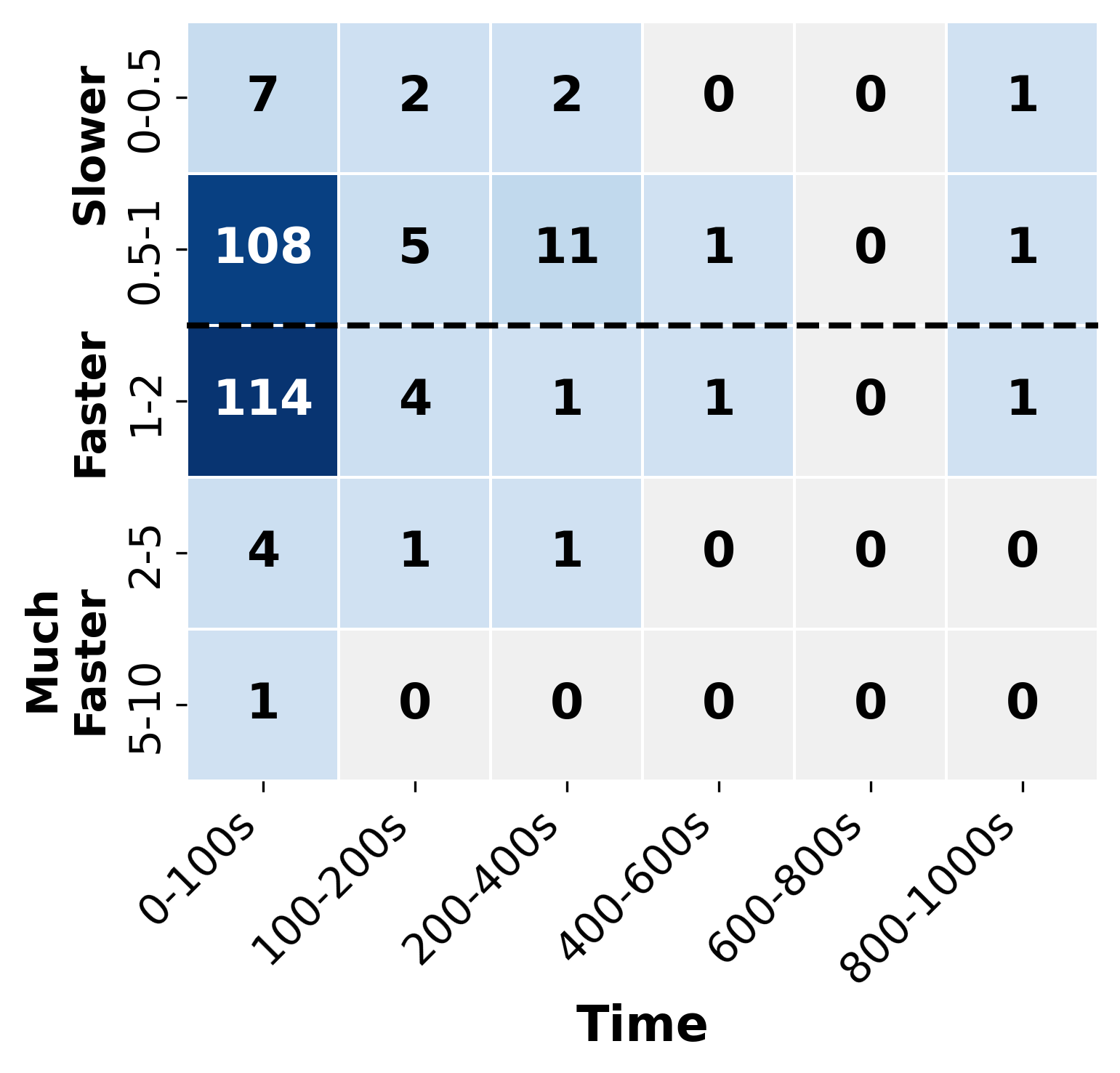}
        \caption{SA attempt 2}
        \label{fig:figure_at2}
    \end{subfigure}
    \begin{subfigure}[b]{0.32\linewidth}
        \includegraphics[width=\linewidth]{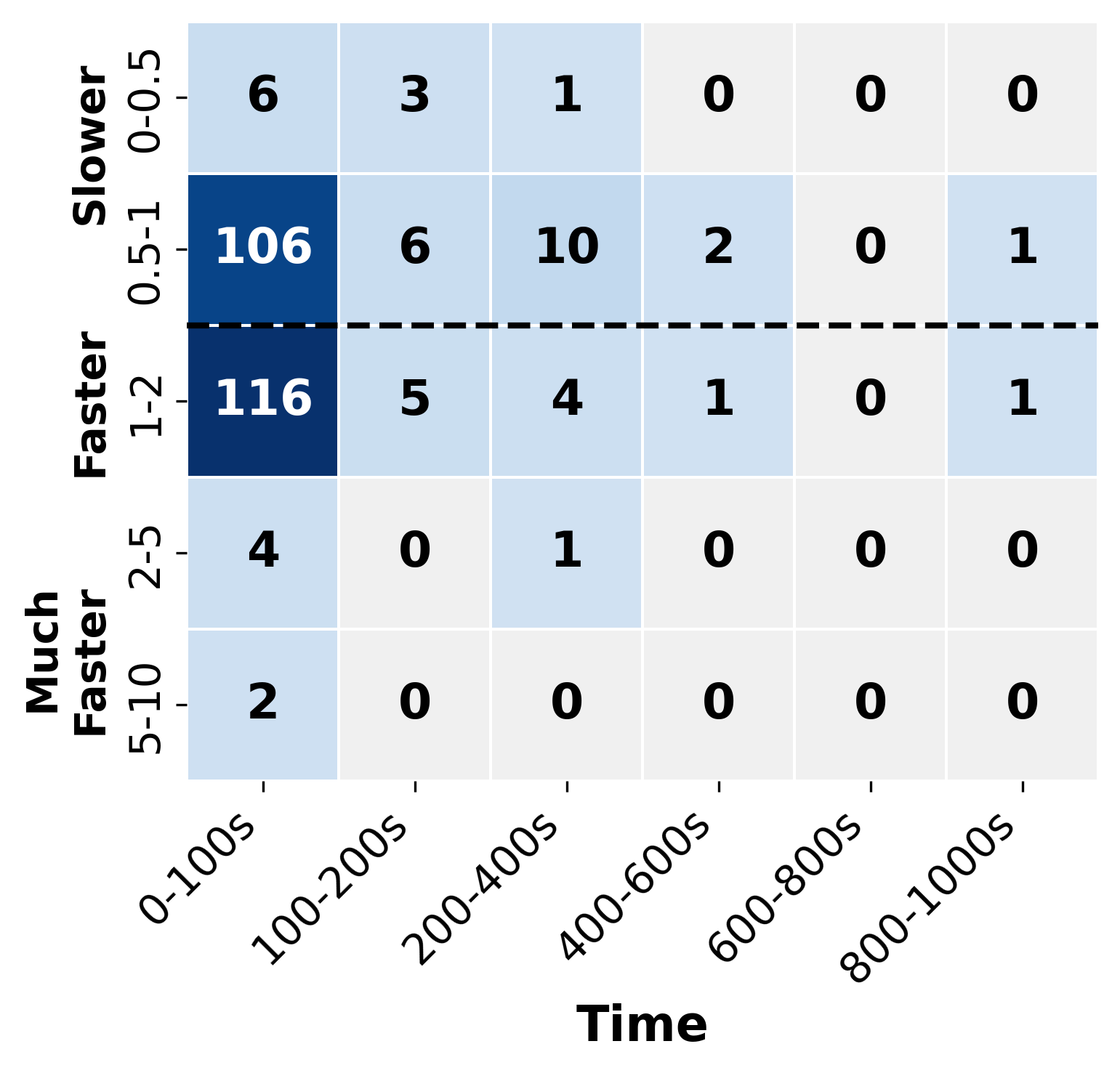}
        \caption{SA attempt 3}
        \label{fig:figure_at3}
    \end{subfigure}
    
    \begin{subfigure}[b]{0.32\linewidth}
        \includegraphics[width=\linewidth]{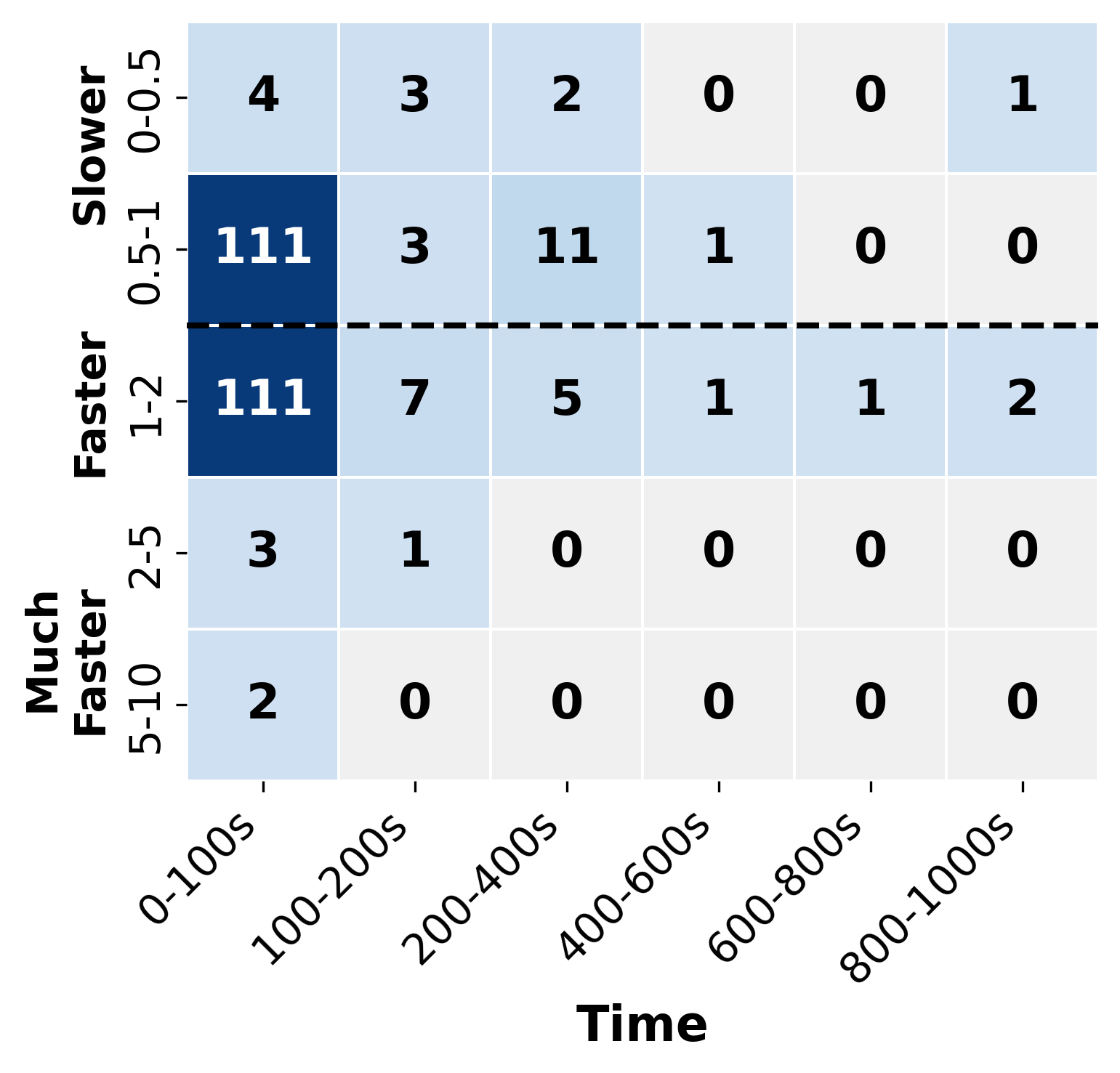}
        \caption{SA attempt 4}
        \label{fig:figure_at4}
    \end{subfigure}
    \begin{subfigure}[b]{0.32\linewidth}
        \includegraphics[width=\textwidth]{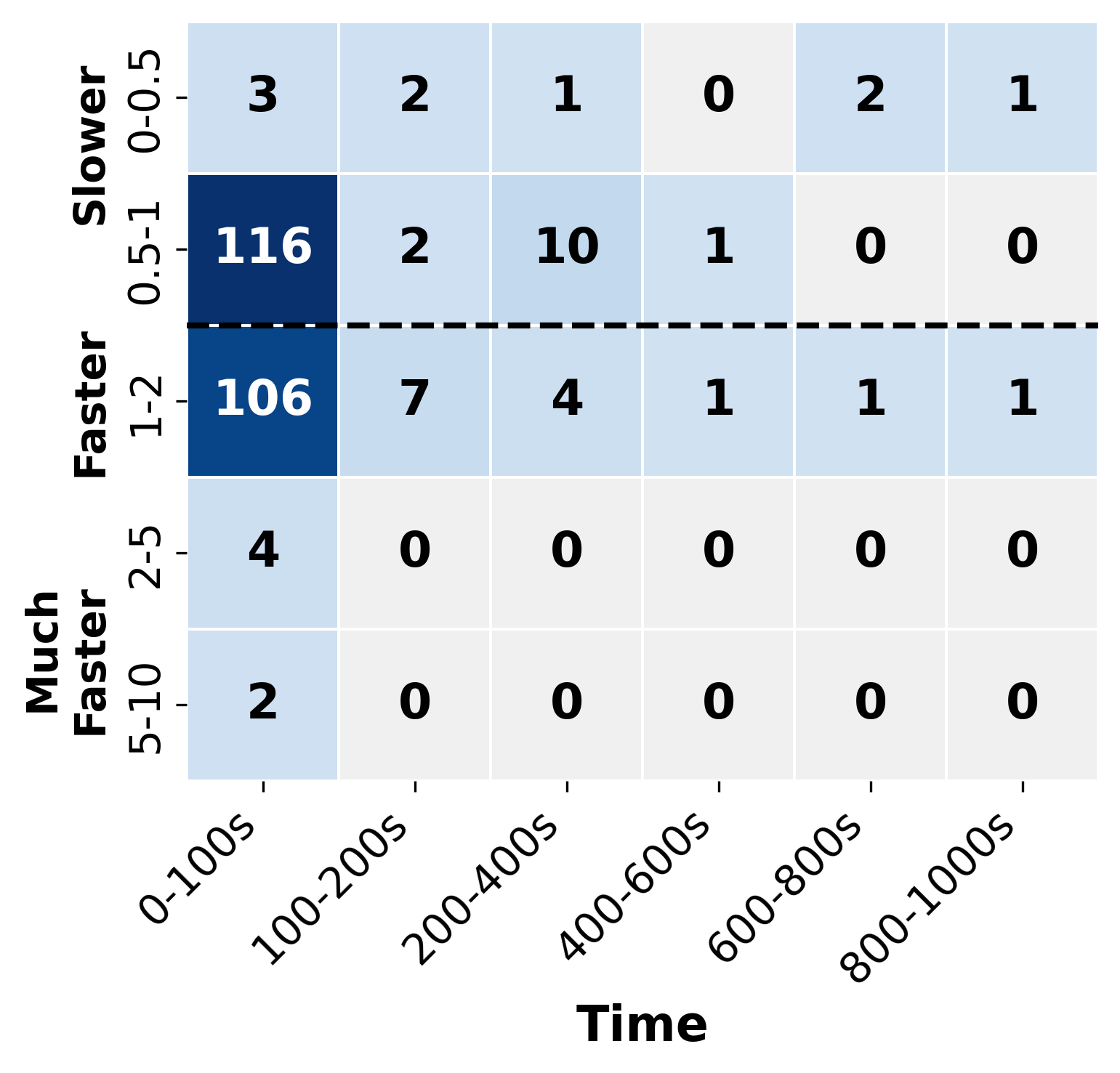}
        \caption{SA attempt 5}
        \label{fig:figure_at5}
    \end{subfigure}
\caption{The distributions of conclusive instances by the variant of \toolb, in different speedups and time costs compared with \toolg. The black dashed line indicates the performance of \toolg: above the black dashed line, \toolg outperforms \toolb; below it, \toolb performs better.}
\label{fig:figure_dis_at}
\end{figure}

Similar to simulated annealing, our tool \toolb is a stochastic approach that is subject to randomness. To analyze the impact of randomness on the performance of \toolb, we conduct five independent attempts with all the verification problems and compare them with the deterministic \toolg. Since \toolb incorporates a simulated annealing approach that makes probabilistic decisions during sub-problem exploration, different runs may explore the verification space in different orders, potentially leading to variations in performance. 

The results are presented in Fig.~\ref{fig:figure_dis_at}, in which we only present the problems that are either solved by \toolg or \toolb. Overall, the performance comparisons between \toolg and \toolb are consistent across different attempts, and are concentrated in the 0.5-2$\times$ speedup bracket. This result suggests that the randomness in the simulated annealing approach does not substantially impact the  effectiveness of \toolb. Notably, in Fig~\ref{fig:figure_at3} and Fig~\ref{fig:figure_at4}, with the nature of the randomness, \toolb exhibits the strength in finding solutions in the time cost of 100-600s that might be difficult to solve by the deterministic approach of \toolg.

\paragraph*{Comparison between \toolb and ABONN~\cite{fukuda2025adaptive}}

\begin{figure}[!tb]
    \centering
    \includegraphics[width=0.65\linewidth]{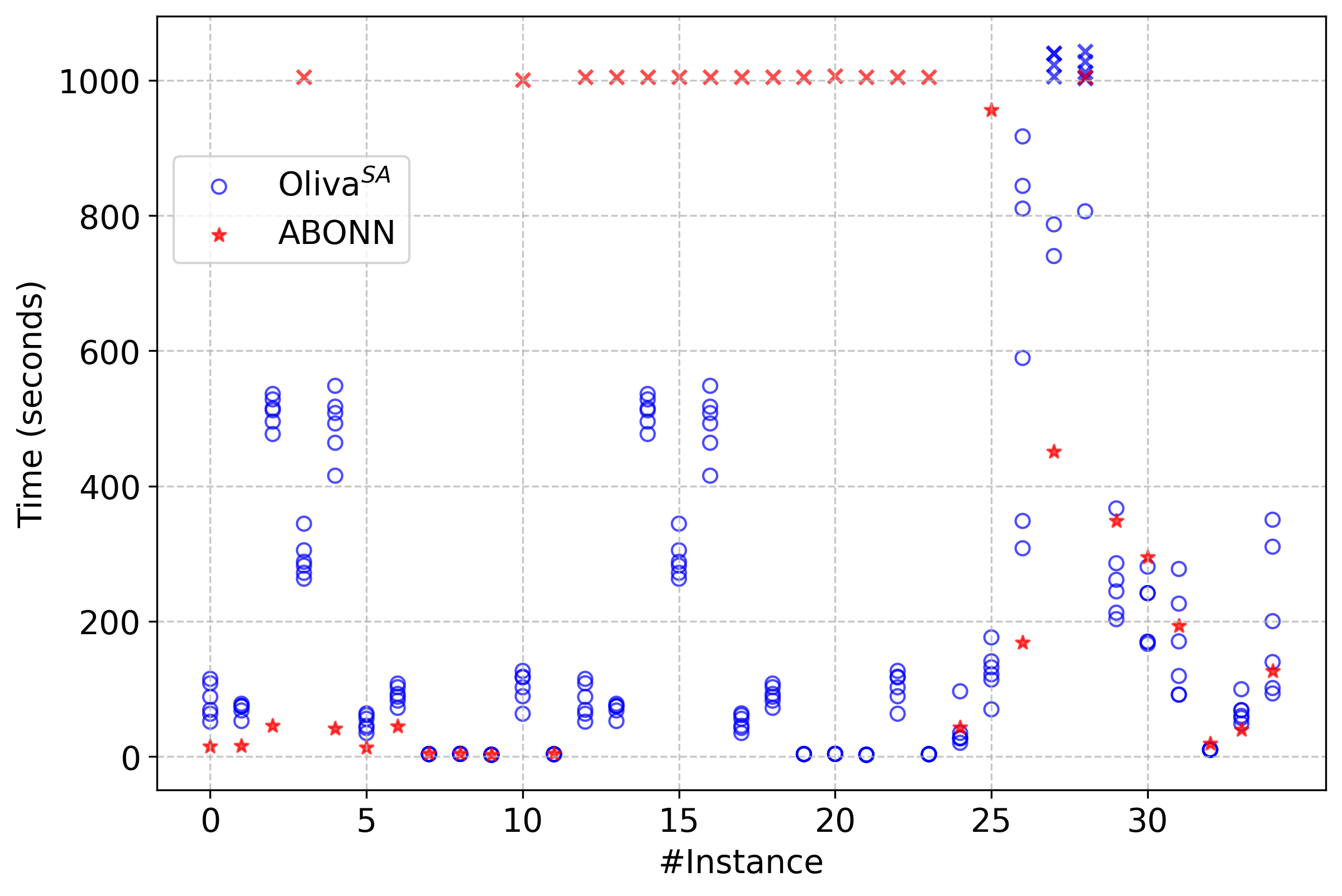}
    \caption{Comparison of deterministic (ABONN) and stochastic (\toolb) approaches. The X-axis represents instance IDs, and the Y-axis shows the corresponding verification time in seconds.}
    \label{fig:samcts}
\end{figure}

We also perform a comparison between our \toolb and MCTS-based ABONN~\cite{fukuda2025adaptive}.  Fig.~\ref{fig:samcts} summarizes the results on 34 conclusive (violated) instances.
As the problem difficulty increases (solving time greater than 300 seconds), \toolb is capable of finding the true counterexample earlier and faster, where in many cases ABONN fails to find a solution (red crosses), while multiple runs of \toolb successfully solve these problems (blue circles). 
Moreover, the stochastic nature of \toolb crucially provides repeated opportunities to find more counterexamples. 
This is particularly true for instance \#28, where we not only repeat the experiments of \toolb, but also allow ABONN to continue running beyond the 1000-second timeout. However, ABONN fails to find any counterexamples, while \toolb succeeds in doing so through repeated runs. This highlights the unique strengths of the stochastic optimisation-based \toolb compared to the deterministic ABONN approach.

\paragraph*{RQ2: How effective is \tool in handling violated and certified problem instances, respectively?}

Fig.~\ref{fig:cmp_vio_safe} shows the performance advantage of \toolg and \toolb over \bab-baseline across different verification tasks. For violated instances (i.e., confirmed counterexample by \bab-baseline and \tool), \toolg and \toolb consistently demonstrate superior efficiency, evidenced by lower median time costs and smaller interquartile ranges across all models.

\begin{figure*}[!tb]
    \centering
    \begin{subfigure}[b]{0.48\linewidth}
        \includegraphics[width=\linewidth]{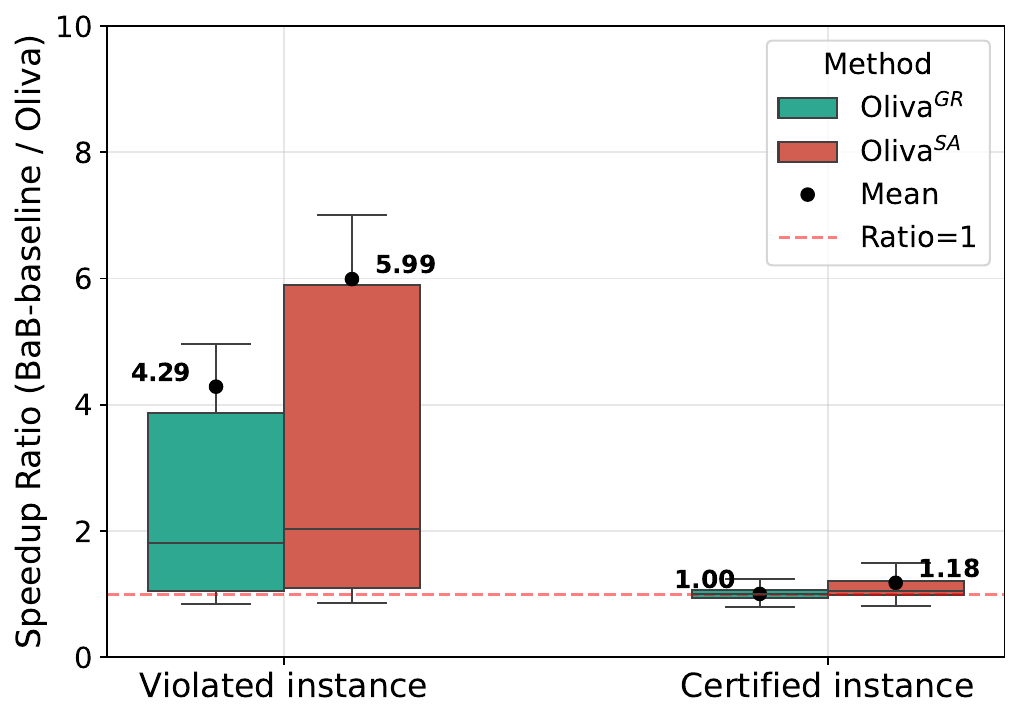}
        \caption{$\mnist_{\ltwo}$}
        \label{fig:boxmnist2}
    \end{subfigure}
    \begin{subfigure}[b]{0.48\linewidth}
    \includegraphics[width=\linewidth]{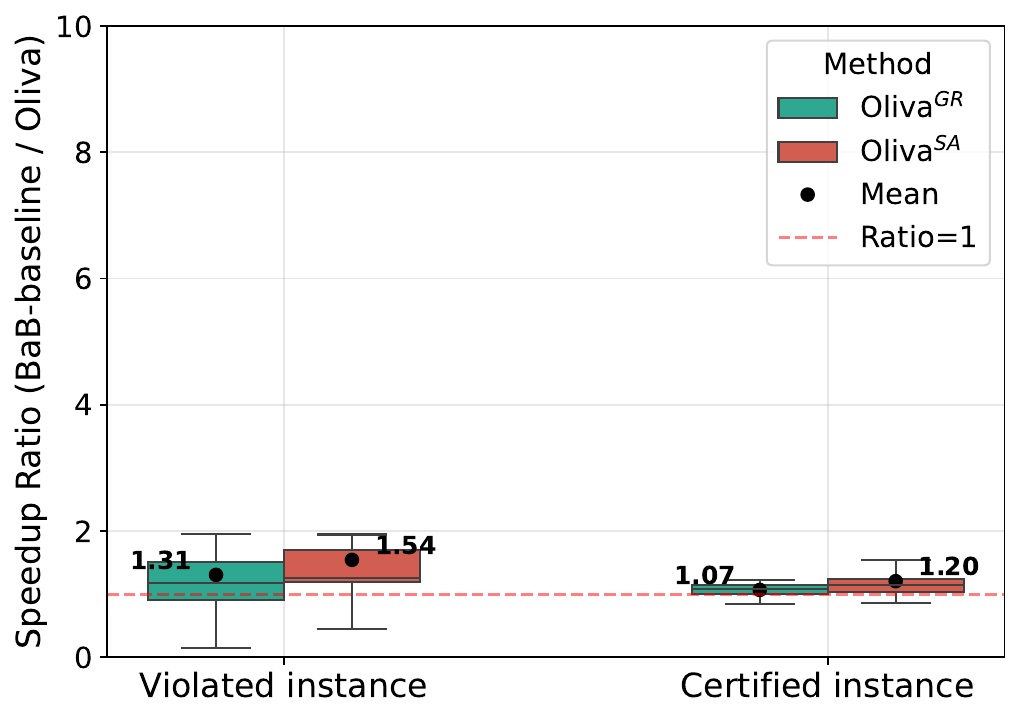}
        \caption{$\mnist_{\lfour}$}
        \label{fig:boxmnist4}
    \end{subfigure}
    
    \begin{subfigure}[b]{0.48\linewidth}
        \includegraphics[width=\linewidth]{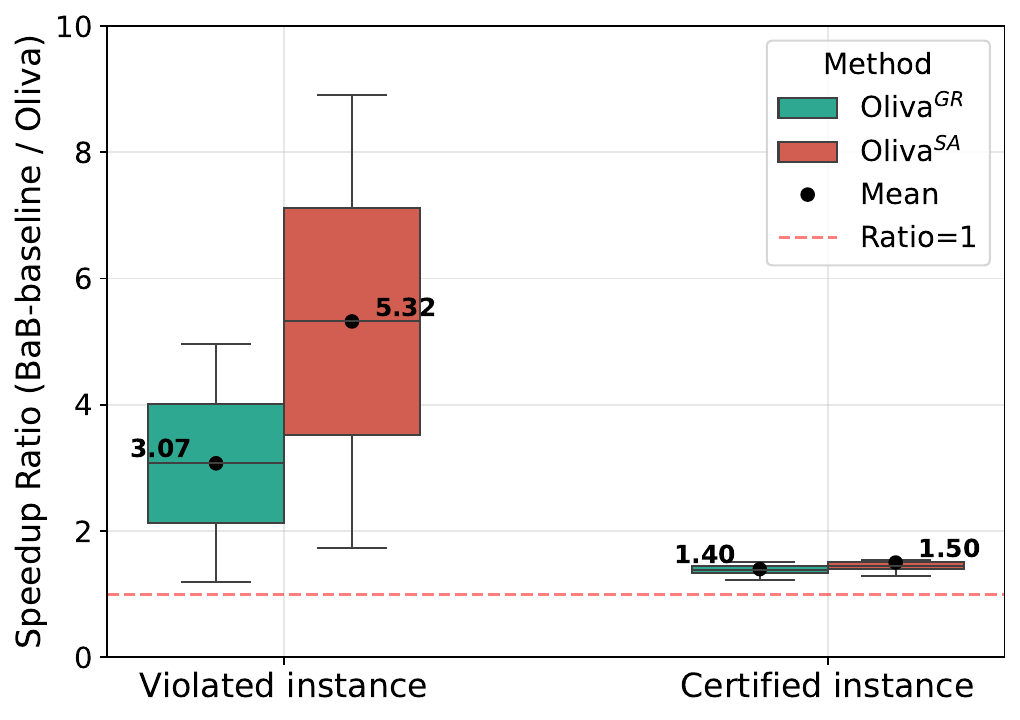}
        \caption{$\OVAL_{{\base}}$}
        \label{fig:sub_base}
    \end{subfigure}
    \begin{subfigure}[b]{0.48\linewidth}
        \includegraphics[width=\linewidth]{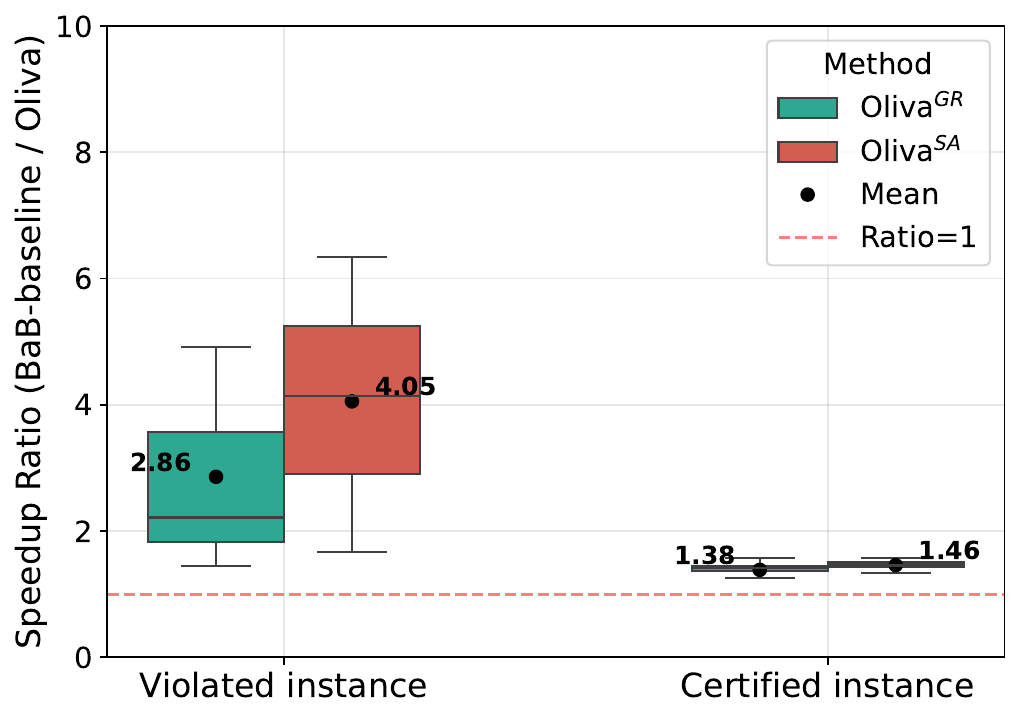}
        \caption{$\OVAL_{{\deep}}$}
        \label{fig:sub_deep}
    \end{subfigure}
    
    \begin{subfigure}[b]{0.48\linewidth}
        \includegraphics[width=\linewidth]{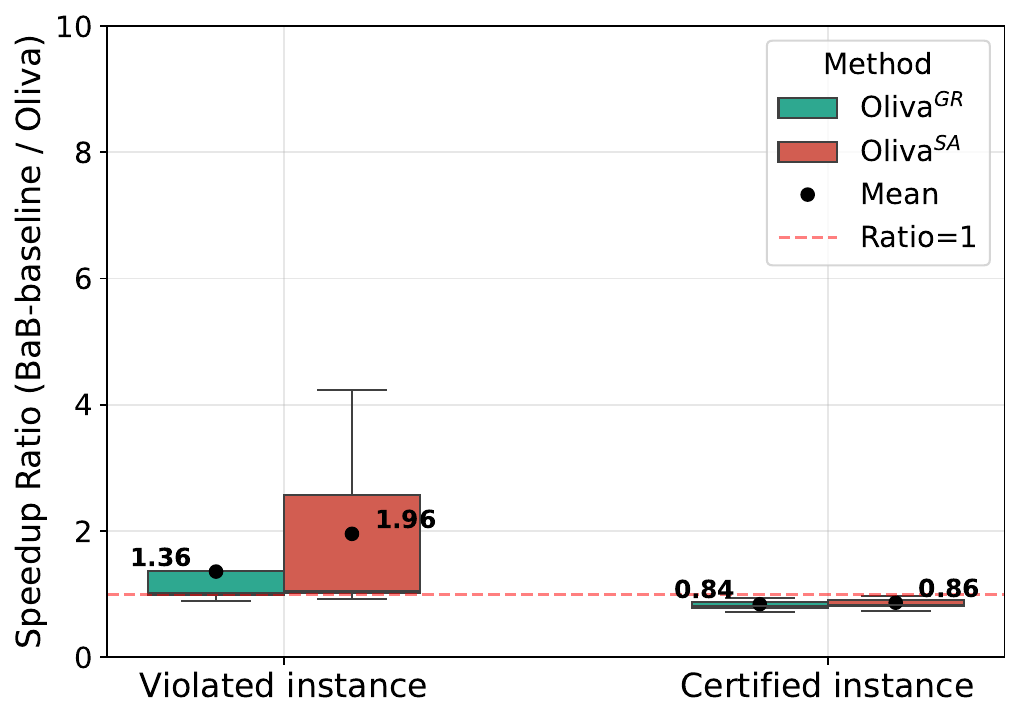}
        \caption{$\OVAL_{{\wide}}$}
        \label{fig:sub_wide}
    \end{subfigure}
    \caption{RQ2 -- Comparison between \bab-baseline and \tool for violated and certified verification problem instances.}
    \label{fig:cmp_vio_safe}
\end{figure*}

In Fig.~\ref{fig:boxmnist2}, \tool shows notably lower median time costs and compressed interquartile ranges compared to \bab-baseline, indicating more consistent and faster execution. In $\OVAL_{{\base}}$ and $\OVAL_{{\deep}}$ (Fig.~\ref{fig:sub_base} and Fig.~\ref{fig:sub_deep}), \toolg and \toolb sub-problem selection strategy is highly effective in problem instances where counterexamples exist. 
For $\OVAL_{{\base}}$, \toolg achieves a median speed-up ratio of approximately 2.5$\times$ for violated instances, and \toolb with reaching up to 4$\times$ faster execution than \bab-baseline. Similarly, for $\OVAL_{{\deep}}$, the median speed-up ratio is around 3$\times$, with peak performance showing up to 5$\times$ acceleration.

In contrast, for certified instances (i.e., confirmed robust), two variants of \tool are generally on par with the \bab-baseline, indicating that in these cases, the performance of \tool is comparable with \bab-baseline, which does not introduce overhead costs. This is expected, because in \bab, most of the time consumption is devoted to the problem solving process for each sub-problem. While our proposed approach introduces little overhead in the selection of the sub-problems to proceed with, the overhead is almost invisible. 

\paragraph*{RQ3: How does the hyperparameter $\lambda$ and $\alpha$ in the order influences the performance of \tool?}

\begin{figure}[!tb]
    \centering
    \includegraphics[width=0.8\linewidth]{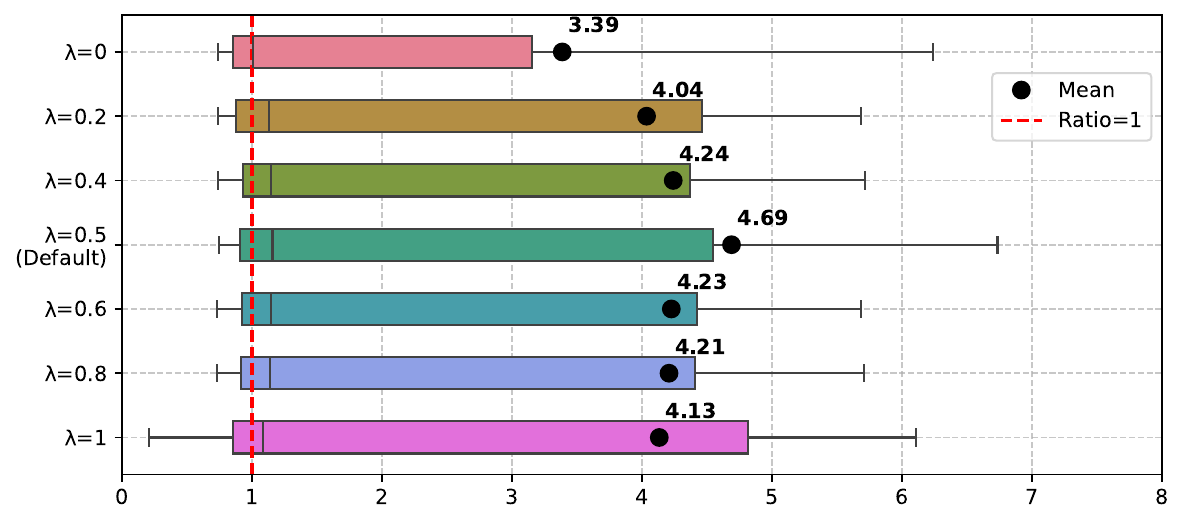}
    \caption{RQ3 -- Speedup over \bab-baseline under different hyperparameter $\lambda$ values on verification tasks with $\OVAL_{\wide}$.}
    \label{fig:heatmap}
\end{figure}

Fig.~\ref{fig:heatmap} highlights the impact of the $\lambda$ parameter on the speedup of \toolg over \bab-baseline for the verification task on $\OVAL_{\wide}$ model. In this experiment, we randomly select 20\% of the instances, including 13 conclusive and 12 unknown cases, based on the \bab-baseline performance. The chosen instances are at varying levels of time consumption. This analysis examines $\lambda$ values ranging from $0.0$ to $1.0$, and summarizes the speedup ratios of \toolg under different  $\lambda$ values. 
The findings reveal that, first, the performances of \toolg under different $\lambda$ are relatively stable, mostly outperforming \bab-baseline. This implies that both of the attributes, including the level of problem splitting and the verifier assessment, are effective in guiding the space exploration. Moreover, the performance slightly improves as $\lambda$ increases from $0.0$ to $0.5$, peaking at a speedup of $4.69$. Beyond $\lambda = 0.5$, speedup declines, indicating $0.5$ as the optimal value. In terms of the average improvement, the speedup vary from $3.39$ to $4.69$ and is not negligible,  which emphasizes the importance of careful tuning of $\lambda$ to maximize the efficiency of \toolg in verification tasks. 

\begin{figure}[!tb]
    \centering
    \includegraphics[width=0.8\linewidth]{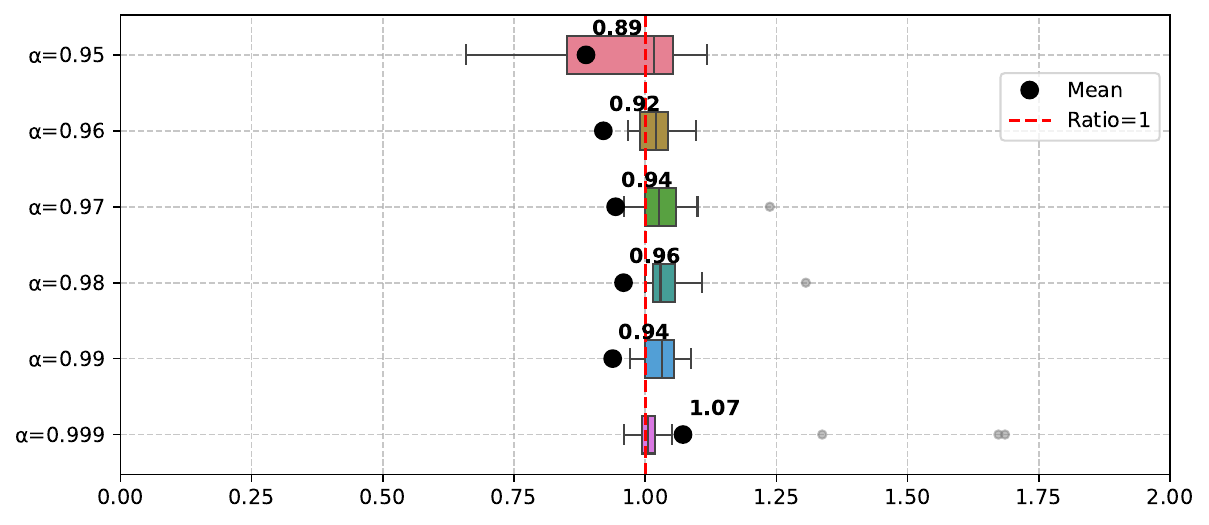}
    \caption{RQ3 -- Speedup over \toolg under hyperparameter $\alpha$ values on verification tasks with $\OVAL_{\wide}$.}
    \label{fig:heatmapalpha}
\end{figure}

Then, we settle $\lambda=0.5$ as the default value to evaluate the temperature reduction parameter $\alpha$, which ranges from 0.95 to 0.999, as suggested by~\cite{kirkpatrick1983optimization}. Fig.~\ref{fig:heatmapalpha} compares the speedup of \toolb over \toolg. 
The hyperparameter $\alpha$  determines the speed of ``temperature'' decreases during the verification processes, directly affecting the balances between exploration and exploitation. First, by the medians, we find that \toolb under all $\alpha$ outperform \toolg. Specifically, with $\alpha=0.999$, \toolb achieves better performance than \toolg in average, which is a very slow cooling schedule ($\alpha$ close to 1) that allows the algorithm to thoroughly explore the search space early on before transitioning to exploitation. As $\alpha$ decreases, the temperature drops more rapidly, giving the algorithm less time to explore before converging on promising areas. The lower performance ratios seen with $\alpha=0.95-0.98$ suggest that this faster cooling schedule may be too aggressive, causing \toolb to commit too quickly to certain branches of the verification tree without adequately exploring alternatives. By our inspection of experimental results, the relative low mean values are caused by some specific cases, for which \toolb does not perform well. Notably, when selecting $\alpha=0.99$, while its mean ratio of 0.94 might seem underwhelming at a first glance, its median performance is notably strong compared to other $\alpha$ values, representing a balanced probability that allows sufficient exploration while maintaining steady progress toward exploitation. In practical verification tasks, as consistency can be as a valuable property,  $\alpha=0.99$ can thus be preferable despite its lower mean ratio.

\section{Related Work}\label{sec:literature}
\paragraph*{Neural Network Verification} Neural network verification has been extensively studied in the past few years, giving birth to many practical approaches~\cite{tjeng2018evaluating,katz2017reluplex,ehlers2017formal,huang2017safety,singh2018deepz,singh2019abstract,singh2018boosting,muller2020neural,wang2018efficient, shi2022efficiently}.
Approximated verifiers are perferable due to their efficiency, and many works aim to seek for tighter bounds for approximation refinement~\cite{anderson1811strong,tjandraatmadja2020convex, singh2019beyond,muller2022prima,muller2021precise,xue2023kprop,ma2024relu, raghunathan2018semidefinite,fazlyab2020safety,dvijotham2020efficient,dathathri2020enabling}.
Some works aim to refine the approximation~\cite{yang2021improving,ostrovsky2022abstraction,hanumanthaiah2024iterative,duong2024harnessing, zhao2022cleverest} by exploiting information from (spurious) counterexamples (known as CEGAR approach). In contrast, our approach does not aim to refine the approximation of the backend verifiers, but we estimate the probability of finding counterexamples in different sub-problems to efficiently explore the sub-problem space.

\paragraph*{Incremental Verification}

Unlike classic verification, incremental verification is an emerging technique that aims to accelerate verification of a neural network $\originNetwork$ by exploiting the existing verification result of an architecture-same and parameter-similar neural network $\originNetwork'$. Here, the existing verification result of $\originNetwork'$ is called a \emph{template}, and different incremental verification approaches~\cite{fischer2022shared,ugare2023incremental, yang2023incremental, olive} differ in the way of defining and exploiting the template. \bab trees have been considered as templates in existing works~\cite{ugare2023incremental, olive}. Ugare et al.~\cite{ugare2023incremental} first leverage the existing verification result for $\originNetwork'$, namely, a \bab tree, to decide whether or not each node in the \bab tree for verification of $\originNetwork$ is needed to expand. Later, Zhang et al.~\cite{olive} extend this framework by attaching quantitative measurements to infer, not only \emph{whether or not}, but also \emph{how much}, each node in the \bab tree for verification of $\originNetwork$ is needed to expand. 
Although the quantitative measurement adopted in~\cite{olive} also considers counterexample potential, its purpose is fundamentally different due to a completely different problem setting, i.e., incremental verification that effectively reuses and utilises the template from $\originNetwork'$ and the BaB tree for verifying $\originNetwork$, and the ordering in~\cite{olive} is primarily used to infer the necessity of visiting different nodes in the template. Incremental verification is an orthogonal but interesting direction, raising the question of how to incorporate existing classic verification techniques designed for from-scratch settings.

\paragraph*{Problem Splitting Strategies in \bab} 

Problem splitting~\cite{shi2024neural,lu2019neural,ferrari2022complete} is an important component in the \bab workflow, which is critical to the efficiency of verification~\cite{bunel2018unified}. Early research considers splitting input space by its dimensions~\cite{wang2018formal,anderson2019optimization,rubies2019fast}; however, that can lead to an exponential growth in the number of sub-problems w.r.t. the number of dimensions, which is intractable especially for applications that have high-dimensional data, such as image classification. More recently, many works start to perform problem splitting by decomposing ReLU functions, which can achieve better performance~\cite{bunel2020branch} and scales better than input domain splitting. In these works, the ReLU selection strategy that decides the next neuron (i.e., the ReLU function) to be split is important and there emerge many effective strategies including DeepSplit~\cite{henriksen2021deepsplit} (which we adopted), BaBSR~\cite{bunel2020branch} and FSB~\cite{de2021improved}. Compared to those works, our approach is different and is orthogonal to them, because instead of selecting ReLU, we select the sub-problems to proceed with, which changes the naive breadth-first strategy of \bab in visiting the sub-problems.

\paragraph*{Sub-Problem Selection}

The idea of guiding verification towards the direction of finding counterexamples is similar to the concept of \emph{falsification}~\cite{corso2021survey} in the domain of hybrid system verification, and has been explored in a number of works in neural network verification. Guo et al.~\cite{guo2021eager} consider the image classification problem, and treat misclassification to different labels as different sub-problems; they design a metric called ``affinity'' to measure the similarity between different labels, and propose a verification approach that prioritizes the sub-problems that have higher affinity, thereby maximizing the possibility of finding counterexamples. Fukuda et al.~\cite{fukuda2025adaptive} also consider the problem of sub-problem selection in the \bab setting. They adopt a rather straightforward search technique called \emph{Monte Carlo tree search (MCTS)}, which features sub-problem selection via \emph{``trial-and-error''}. However, the MCTS approach in~\cite{fukuda2025adaptive} involves a fixed policy of tree exploration, and once the counterexample search heuristics does not work well, the performance can be diminished. In comparison, our strategy \toolb adopts stochastic optimization and can benefit from our algorithmic stochasticity, namely, even though an attempt of tree exploration is not effective, by using a different random seed, we can explore the tree in a different way, which may give the chances of finding counterexamples.

\paragraph*{Testing and Attacks} 
Testing is another effective quality assurance for neural networks, and it has also been extensively studied, such as~\cite{pei2017deepxplore, odena2019tensorfuzz,gao2020fuzz, tian2018deeptest}. 
A similar line of works are adversarial  attacks, that aim to generate adversarial examples~\cite{goodfellow2015explaining,andriushchenko2020square,xie2019improving,zhang2022branch,carlini2017towards} to fool the neural networks. While these approaches are efficient, they cannot provide rigorous guarantee on specification satisfaction, when they cannot find a counterexample, because they solve the problem by checking infinitely many single inputs in the input space. In comparison, our approach is sound, in the sense that if there does not exist a counterexample in the network, we can finally reach the certification of the network. This is because by our problem splitting, we deal with finitely many sub-problems.

\section{Conclusion and Future Work}\label{sec:conclusion}
    
In this paper, we propose an approach \tool with two variants \toolg and \toolb in order to achieve high efficiency of neural network verification. \tool introduces a severity order over the sub-problems produced by \bab. This order, called counterexample potentiality, estimates the suspiciousness of each sub-problem in the sense how likely it contains a counterexample, based on both the level of problem splitting and the assessment from approximated verifiers that signify how far a sub-problem is from being violated. By prioritizing the sub-problems with higher counterexample potentiality, \tool can efficiently reach verification conclusions---either finding a counterexample quickly or certifying the neural network without significant performance degradation. Specifically, our two variants of \tool implement different strategies: $i$) \toolg greedily exploits the most suspicious sub-problems; $ii$) \toolb, inspired by simulated annealing, strikes a balance between exploitation of the promising sub-problems and exploration of the sub-problems that are less promising. Our experimental evaluation across 5 neural network models and 690 verification problems demonstrates that \tool can achieve up to 80$\times$ speedup, compared to the state-of-the-art baseline approaches. Moreover, we also show the breakdown comparison for the certified problems and for the falsified problems, and we demonstrate that the performance advantages of the proposed approach indeed come from the strategy guided by counterexample potentiality.

To the best of our knowledge, our work is the first that exploits the power of stochastic optimization in neural network verification, by taking the sub-problems as the search space. In future, we plan to do more exploration in this direction, e.g., trying other stochastic optimization algorithms, such as genetic algorithms, and comparing the performances of different strategies, in order to deliver better verification approaches.

\bibliography{bib2doi}

\end{document}